\documentclass[11pt,a4paper,leqno]{article}

\usepackage[utf8]{inputenc}
\usepackage[T1]{fontenc}

\usepackage{amsmath}
\usepackage{amsthm}
\usepackage{amssymb}
\usepackage{authblk}

\usepackage[bf,labelsep=period]{caption}
\usepackage{covington}

\usepackage[margin=3.0cm]{geometry}
\usepackage{graphicx}

\usepackage[hidelinks]{hyperref}
\usepackage[backend=biber,style=apa,autocite=inline]{biblatex}
\newtheorem{lemma}{Lemma}
\newtheorem{prop}{Proposition}

\usepackage{sectsty}
\sectionfont{\centering\bfseries\large}
\subsectionfont{\centering\rm\itshape\normalsize}
\newcommand{\nothing}{\varnothing}

\title{A bifurcation threshold for contact-induced language change}
\author{Henri Kauhanen}
\affil{University of Konstanz}
\date{\normalsize 18 August 2022}
\bibliography{master.bib}

\begin{document}

\maketitle

\begin{abstract}
One proposed mechanism of language change concerns the role played by second-language (L2) learners in situations of language contact. If sufficiently many L2 speakers are present in a speech community in relation to the number of first-language (L1) speakers, then those features which present a difficulty in L2 acquisition may be prone to disappearing from the language. This paper presents a mathematical account of such contact situations based on a stochastic model of learning and nonlinear population dynamics. The equilibria of a deterministic reduction of the model, describing a mixed population of L1 and L2 speakers, are fully characterized. Whether or not the language changes in response to the introduction of L2 learners turns out to depend on three factors: the overall proportion of L2 learners in the population, the strength of the difficulty speakers face in acquiring the language as an L2, and the language-internal utilities of the competing linguistic variants. These factors are related by a mathematical formula describing a phase transition from retention of the L2-difficult feature to its loss from both speaker populations. This supplies predictions that can be tested against empirical data. Here, the model is evaluated with the help of two case studies, morphological levelling in Afrikaans and the erosion of null subjects in Afro-Peruvian Spanish; the model is found to be broadly in agreement with the historical development in both cases.

\bigskip

  \noindent\textbf{Keywords:} language change; language complexity; language contact; learning models; nonlinear dynamics; sociolinguistic typology
\end{abstract}

\section{Introduction}
\label{sec:introduction}

The role played by second-language (L2) speakers in language change has attracted increasing interest in recent years, uniting research traditions in historical linguistics \autocite{Wee1993}, sociolinguistic typology \autocite{Tru2004,Tru2010,Tru2011}, language complexity \autocite{Kus2003,Kus2008,McWho2011} and second-language acquisition \autocite{SorSer2009,BerSem2022}. If L2 learners generally struggle to acquire certain linguistic features in a target-like fashion, and if L2 learners are sufficiently prevalent in a speech community in a particular historical situation, then features of this kind, namely L2-difficult features, may be expected to be lost from the language over extended time. This mechanism---change triggered by L2-difficulty---is supported both by empirical studies in L2 acquisition and artificial language learning \autocite[e.g.][]{MarBel2006,BerSem2022} and by large-scale typological studies which suggest a negative correlation between the proportion of L2 speakers and language complexity \autocite{LupDal2010,BenWin2013,SinDiGar2018,Sin2020}.

Although quantitative laboratory and typological studies are essential in supplying the empirical content of this hypothesis, many open questions remain about the large-scale population dynamics of this proposed mechanism of contact-induced change. For instance, just how many L2 speakers does it take to tip the balance in a contact situation? Clearly, a small minority of L2 learners is generally insufficient to lead to wholesale change. Under what conditions, then, do those who speak the language as a first language (L1) also adopt the new (non-L2-difficult) variety? Such questions can be tackled in a mathematical modelling approach, in which predictions can be generated and the range of possible behaviours of dynamical systems understood in full detail.

This paper presents a first sketch of such an approach, beginning from simple and general assumptions about language learning and social interaction. Building on the variational model of language acquisition and change \autocite{Yan2002}, I propose an extension of this model which subsumes both L1 and L2 learning under the same general mechanism. The extended model assumes that L1 and L2 learning are fundamentally alike processes, in the sense that they both depend on general mechanisms of reinforcement \autocite{BusMos1955}. In the extended model, however, adult L2 learning is potentially disadvantaged in the sense that these learners face a difficulty in the case of certain linguistic features acquired unproblematically by children. To study the population-dynamic consequences of L1 and L2 learners interacting in the same population, the extended model considers a mixed speech community. This yields testable predictions about the eventual fates of linguistic features exposed to contact situations of different kinds, including predictions about the stable states attained, conditioned on model parameters such as the proportion of L2 speakers in the population and the relative strength of the L2-difficulty of the linguistic variable involved.

The model allows a number of qualitatively different evolutionary outcomes: complete extinction of the L2-difficult feature, temporary loss of the L2-difficult feature followed by its reacquisition, or stable variation between two grammars, one of which carries the L2-difficult feature and one which does not. These theoretically predicted outcomes correspond to different patterns observed---to varying degrees of detail---in the empirical record. To test the model, I apply it to two empirical case studies: the deflexion of the verbal paradigm in Afrikaans \autocite{Tru2011} and the erosion of null subjects in Afro-Peruvian Spanish \autocite{SesGut2018}. Model parameters are estimated from empirical data where possible, and intuitive arguments for plausible orders of magnitude are provided where direct empirical calibration is impossible. The model is found to correctly characterize the broad outlines of both empirical developments, although challenges arise both from the lack of relevant empirical data and from a number of idealizing assumptions that must be made in order to simplify the model's mathematics in the interest of retaining analytical tractability.

It should be emphasized that the model presented and analysed in this paper is concerned with only one potential external (``environmental'') predictor of contact-induced change, namely the proportion of L2 speakers in a speech community. Other possible external factors conditioning language change have been considered in the literature, ranging from population size \autocite{LupDal2010,Net2012,Kop2019} to social network geometries \autocite{KeGonWan2008,FagEtal2010,Kau2017,JosEtAl2021} or combinations of such parameters \autocite{Tru2004}. Although the present model sets such effects aside, this is not to deny their importance. Indeed, further variables may be introduced to the modelling framework whose broad outlines are laid down in this paper, in future work. It is quite likely that in a complex process such as language change, multiple factors interact in intricate ways. In fact, this is already the case for the factors considered in the present paper, L2 speaker proportion (a demographic parameter), strength of L2-difficulty (a psychological parameter) and parsing advantage (a linguistic parameter), as will be discussed at length below.

The paper is structured as follows. Section \ref{sec:explananda} provides a brief description of the two empirical test cases. Language learning, both L1 and L2, is discussed in Section \ref{sec:learning}, with the goal of characterizing the terminal state of a population of learners after a long learning period in a stationary random learning environment. This terminal state is then utilized to define inter-generational population dynamics in Section \ref{sec:population-dynamics}, in which the equilibria of the resulting nonlinear system of equations are studied. Relevant linguistic and demographic parameters are calibrated in Section \ref{sec:application} using available empirical data, followed by application of the mathematical model. Section \ref{sec:concluding-remarks} concludes, discusses some of the limitations of the approach, and provides a few directions for future research. To streamline discussion throughout the paper, most mathematical derivations are collected in an Appendix.

\section{The explananda}
\label{sec:explananda}

\subsection{Afrikaans verbal deflexion}

Germanic languages typically exhibit complex verbal morphology, with numerous forms found throughout the verbal paradigms, often without a transparent one-to-one mapping between form and meaning:
\begin{example}
  \label{ex:laufen}
  \emph{German}\\ laufen `to run'

  \begin{tabular}{lll}
    & singular & plural \\
    1st & laufe & laufen \\
    2nd & läufst & lauft \\
    3rd & läuft & laufen
  \end{tabular}
\end{example}
\begin{example}
  \label{ex:lopen}
  \emph{Dutch}\\ lopen `to run/walk'

  \begin{tabular}{lll}
    & singular & plural \\
    1st & loop & lopen \\
    2nd & loopt & lopen \\
    3rd & loopt & lopen
  \end{tabular}
\end{example}
In this respect Afrikaans, which began as a contact variety of Dutch in the Cape Colony in the 17th and 18th centuries, behaves rather differently. Modern Afrikaans has only one form throughout the paradigm (corresponding to the stem of the Dutch verb):
\begin{example}
  \label{ex:loop}
  \emph{Afrikaans}\\ loop `to walk'

  \begin{tabular}{lll}
    & singular & plural \\
    1st & loop & loop \\
    2nd & loop & loop \\
    3rd & loop & loop
  \end{tabular}
\end{example}
This deflexion of verbal paradigms is just one manifestation of a larger-scale morphological regularization that separates Afrikaans from its ancestor \parencite[see][]{Pon1993}.

The specific and rather unique setting in which Afrikaans arose has prompted scholars to discuss its origin extensively. There are three major competing theories: the superstratist hypothesis, according to which the structural features of Afrikaans arose from Dutch in a process of internal development; the interlectalist hypothesis, according to which the structure of Afrikaans is explained by competition between multiple different dialects of Dutch that were brought to the Cape; and the creolist hypothesis, according to which Afrikaans is a creole or semicreole that arose from the interaction between the colonizers, the native Khoekhoe population, and slaves brought by the Europeans from other parts of Africa and parts of Asia \autocite{Rob2002}. It is clear, however, that contact must have played some role in the formation of the language; the specific classification of Afrikaans as a creole, semicreole or other kind of contact variety is less relevant here.

If the three main population constituents of the Dutch Cape Colony were in extensive linguistic contact, as seems likely \autocite{Pon1993,Rob2002}, then an amount of L2 learning will have taken place: both the native Khoekhoe and the imported slaves would have had to learn at least a limited amount of Dutch to communicate with the colonizers. The (adult) L2 Dutch spoken in the colony would form part of the input of the following generations of L1 learners, and through this mechanism of nativization of L2 output, changes could stabilize as part of the language that was in the process of formation \autocite{Tru2011}.\footnote{On the loose-knit nature of early Dutch Cape society, \textcite[3]{Pon1993} writes: ``there was a lack of community structure such as that which characterised the early New England settlement, where not only families but entire communities crossed from England to America. Cape society reconstituted itself beyond the purview of the VOC [Vereenigde Oostindische Compagnie, the Dutch East India Company]: new, extended families crystallised from what was at first a motley crowd of men with only a few women, and a speech community came into being where Dutch was in competition with other languages.'' Three observations are particularly important to understand the sociolinguistic situation of early Cape Colony. First, the Dutch East India Company explicitly required Dutch (and not, for instance, Portuguese) to be used in communication with slaves \parencite[25]{Pon1993}. Secondly, it is suspected that childcare was in many cases left to the responsibility of slave or Khoekhoe women, who ``transferred to them their own approximate (broken) Dutch'' \parencite[8]{Pon1993}. Finally, the Cape was not a plantation colony but rather one in which slaves were distributed relatively uniformly across the population \parencite[12]{Pon1993}.} Complex meaning-to-form mappings are thought to be difficult for L2 learners generally \autocite{DeKey2005}, but experimental evidence lends an additional degree of credibility to this general idea in the specific case of the Dutch verb. \textcite{Blo2006} compared child L1, child L2 and adult L2 learners of Dutch for knowledge of Dutch verbal morphology by way of a sentence completion task; the L2 learners had either Turkish or Moroccan backgrounds and spoke Turkish, Moroccan Arabic or Tarifit (a Zenati Berber language of Morocco) as their L1. Compared to Dutch L1 learners, who exhibited an accuracy of 96\% standard use of the verbal paradigm, child L2 learners showed 83\% (Turkish) or 85\% (Moroccan) accuracy. The adult L2 learners, however, attested only 57\% (Turkish) or 56\% (Moroccan) overall accuracy. This suggests that verbal morphology is an L2-difficult feature to acquire (irrespective of the learner's L1) and may thus favour paradigmatic levelling (regularization) when a significant number of adult L2 learners are involved in the contact situation.

\subsection{Erosion of null subjects in Afro-Peruvian Spanish}
\label{sec:afro-peruvian-spanish}

Consistent null subject (NS) languages such as standard forms of Spanish and Italian exhibit the omission of non-emphatic, non-contrastive referential subject pronouns in finite clauses \autocite[see][]{RobHol2010}, illustrated here for Spanish:
\begin{example}
  \label{ex:hablo-espanol}
  \digloss{$\nothing$ Hablo español.}{{} speak Spanish}{I speak Spanish.}
\end{example}
In some contexts, such as with expletive subjects, a null pronoun is strictly obligatory:
\begin{examples}
\item \digloss{$\nothing$ Llueve.}{{} rains}{It's raining.}
\item \label{ex:el-llueve} \digloss{*Ello llueve.}{it rains}{It's raining.}
\end{examples}
An overt referential subject is, however, required under certain circumstances, and the use of null vs.~overt pronouns in general is typically conditioned by complex semantic and pragmatic considerations. Consider the following examples \autocite[via][226]{Mon2004}:
\begin{examples}
\item\label{ex:montrul-1} \digloss{Nadie\textsubscript{i} dice que $\nothing$\textsubscript{i/j} ganará el premio.}{no-one says that {} {will win} the prize}{No-one says that he will win the prize.}
\item\label{ex:montrul-2} \digloss{Nadie\textsubscript{i} dice que él\textsubscript{*i/j} ganará el premio.}{no-one says that he {will win} the prize}{No-one says that he will win the prize.}
\end{examples}
With a null pronoun, the understood subject of the embedded clause in this example may or may not be coreferential with the subject of the main clause \eqref{ex:montrul-1}, whereas with an overt pronoun the referents of the subjects must differ: the interpretation in which the pronoun is bound is strictly ungrammatical if the pronoun is overt \eqref{ex:montrul-2}. Evidence now exists that features of this kind---ones that hinge on the syntax--semantics and syntax--pragmatics interfaces---are difficult for L2 learners to acquire (\citeauthor{SorSer2009} \citeyear{SorSer2009};~see also \citeauthor{WalBre2019} \citeyear{WalBre2019} for a theoretical account of the diachrony of null subjects as loss/gain of uninterpretable features).

Rates of subject expression differ across varieties of Spanish, with American varieties exhibiting significantly higher rates of overt subjects compared to Peninsular Spanish: \textcite[195]{MarSanPhD} reviews the available literature and shows that these rates range from as low as 12\% in Valladolid, Spain, to 60\% in San Juan, Puerto Rico. In particular, it has been suggested that a number of Afro-Hispanic Languages of the Americas (AHLAs)---languages that emerged from contact between Spanish and African languages in the Americas in colonial settings---attest mixed\footnote{I use the term \emph{mixed} rather than ``partial'' \autocite{SesGut2018} here, as the latter term has taken on a specific theoretical meaning in the generative literature on null subjects \autocite[see][]{RobHol2010}. It is unclear, and from the point of view of the present study of secondary interest, whether Afro-Peruvian Spanish is a partial NS language in the latter sense; the relevant fact is that the language lies, \emph{in some sense}, between the polar extremes of being a consistent NS language and being a (consistent) non-NS language.} pro-drop systems \autocite{SesGut2018}. These languages not only employ overt subjects where Peninsular Spanish uses null subjects \eqref{ex:sessarego-1}, they may also exhibit features of both NS and non-NS languages simultaneously \eqref{ex:toribio}:
\begin{examples}
\item\label{ex:sessarego-1}
  \emph{Afro-Peruvian Spanish} \autocite[384]{Ses2014}
  \digloss{Paco fue a casa. Él se tomó una botella de cerveza y después él se fue al bar de fiesta.}{Paco went to home he himself took a bottle of beer and afterwards he himself went {to the} bar of party}{Paco went home. He drunk a bottle of beer and then went to the bar to have fun.}
\item\label{ex:toribio}
  \emph{Dominican Spanish} \autocite[319]{Tor2000}
  \digloss{Yo no lo vi, él estaba en Massachusetts, $\nothing$ acababa de llegar.}{I not him saw he was in Massachusetts {} finished of arrive}{I did not see him, he was in Massachusetts, he had just arrived.}
\end{examples}
One way of making sense of such data is to hypothesize that speakers have access to more than one grammar simultaneously \autocite{Kro1989,Yan2002}---in this case, to both a NS and a non-NS version of Spanish---and employ these competing grammars variably and probabilistically (\citeauthor{Tor2000} \citeyear{Tor2000}; \citeauthor{SesGut2018} \citeyear{SesGut2018}; for an extensive variationist account of subject pronoun expression in Spanish--English bilinguals in New Mexico, see \citeauthor{TorTra2018} \citeyear{TorTra2018}).

Empirical learning data are again relevant in suggesting explanations for the observed patterns. \textcite{MarBel2006} tested Greek adult L2 learners of Spanish on the expression of NS pronouns; the results demonstrate that L2 learners tend to overuse overt pronouns compared to native speakers, even when the L1 is a NS language. Specifically, in a cloze task in which participants had to either express or omit a subject pronoun, intermediate L2 learners employed null subjects 52\% of the time in matrix clauses and 81.66\% of the time in subordinate clauses, compared to 85.50\% and 98.13\% in advanced learners, and 96.00\% and 100\% in native controls \autocite[92]{MarBel2006}. These findings suggest that L2 learners---with the possible exception of those at an advanced proficiency level---struggle with the precise pragmatic conditioning of the null/overt contrast in a NS language \autocite[for converging evidence from different language pairings, see][]{Bin1993,PerGla1999}. They thus support the notion that a contact situation involving (imperfect) L2 learning followed by L1 nativization may help to explain the diachronic erosion of null subjects.

In particular, such an explanation is plausible in the case of Afro-Peruvian Spanish, a variety spoken mainly in the coastal regions of modern Peru \autocite{Ses2014,SesBook,SesGut2018}. Afro-Peruvians---descendants of slaves who were brought from various parts of Africa and forced to work on plantations, in mines and as servants in cities from the 16th century onwards---amounted to 3.6\% of Peru's population in the most recent, 2017 census \autocite[222]{PeruCensus}. Although Peru abolished slavery formally in 1854, most of the Afro-Peruvian population lived in relative poverty under a semi-feudal system until as recently as the second half of the 20th century \autocite[79]{SesBook}. The probable linguistic consequences of these sociohistorical facts will be discussed in detail in Section \ref{sec:application}.

\subsection{Desiderata}

Previous research thus suggests that the deflexion of the Afrikaans verbal paradigm and the erosion of null subjects in Afro-Peruvian Spanish may both be traced back to an earlier contact situation involving significant amounts of L2 learning. Apart from explicating what a ``significant amount'' means in this context, thereby providing a unified account of what is common to both cases, a mathematical modelling approach needs to be able to account for the \emph{differences} between the two situations. In the case of Afrikaans, the loss of verbal morphology is complete: all verbs are reduced to one form throughout the paradigm. In the case of Afro-Peruvian Spanish, and several other AHLAs, the loss of null subjects is incomplete, in the sense that these languages attest features of both NS and non-NS languages. Moreover, in the case of Afro-Peruvian Spanish at least, there is evidence of the younger generations moving in the direction of the NS grammar again, suggesting that the mixed NS status may not be stable; fieldwork interviews suggest that the Afro-Peruvian variety could be lost in favour of standard Peruvian Spanish ``in two generations, or maybe only one'' \autocite[397]{Ses2014}. In the ideal situation, a modelling approach would predict in what circumstances---under what combinations of model parameters---each development is likely to unfold.

\section{Two learning algorithms}
\label{sec:learning}

At a very general level, learning can be characterized as a stochastic process of modification of a distribution of probabilities to act in specific ways. These modifications are prompted by the outcomes of previous actions in a learning environment which supplies feedback to the learner \autocite{BusMos1955}. The variational learning (VL) framework \autocite{Yan2002} constitutes a linguistic interpretation of this general statement. In the simplest case, the learner must make a binary choice between two options, such as between employing null and overt subject pronouns. The learner stores a probability of use of one of these options, $p$; learning consists of modifications to $p$ in response to previous interactions between learner and environment by way of a set of operators, which mathematically speaking are simply functions applied to $p$ to transform its value.

Most applications of VL to date have assumed Bush and Mosteller's (\citeyear{BusMos1955}) one-dimensional linear reward--penalty learning scheme \autocite[e.g.][]{Yan2000,HeyWal2013,IngLegYan2013,Dan2017,KauWal2018,SimEtAl2019}. This learning algorithm makes use of two linear operators, $f$ and $g$, defined as follows:
\begin{equation}
  \label{eq:LRP}
    \begin{aligned}
      f(p) &= p + \gamma (1-p), \\
      g(p) &= p - \gamma p,
    \end{aligned}
\end{equation}
where $p$ is the probability of grammatical option $G_1$, $1-p$ is the probability of grammatical option $G_2$, and $0 < \gamma < 1$ is a learning rate parameter. Operator $f$ is applied if the learner chooses $G_1$ and this choice manages to parse the input received by the learner at that learning step. From the form of the operator, it is easy to see that application of this operator constitutes a reward to $G_1$: the next time the learner has to make a choice between $G_1$ and $G_2$, they are more likely to choose $G_1$ than before. On the other hand, if $G_1$ is selected but does not parse the input received by the learner, operator $g$ is applied, disfavouring this grammatical option, meaning that the value of $p$ is decreased. Should the learner choose $G_2$, the logic of application of the operators is flipped (with only two options, rewarding $G_1$ is tantamount to punishing $G_2$, and vice versa).

Linear reward--penalty learning has a number of attractive mathematical properties, the most important of which have to do with the expected eventual state attained by the learner. Since learning is a stochastic process, it is impossible to predict the evolution of $p$ over time exactly (unless we have full knowledge of the sequence of the learner's choices and the learning environment's responses at each learning step---but this is impossible except in a strict laboratory setting). However, both the expected (mean) value of $p$ and its variance admit explicit solutions (see \citeauthor{BusMos1955} \citeyear{BusMos1955} and the Appendix). With increasing learning iterations the expectation, denoted by $\langle p \rangle$, eventually tends to the asymptotic value
\begin{equation}
  \label{eq:LRP-asymptote}
  \langle p \rangle_{\infty} = \frac{\pi_2}{\pi_1 + \pi_2}
\end{equation}
in the limit of infinite learning iterations. Here, $\pi_1$ and $\pi_2$ are the probabilities with which the environment punishes $G_1$ and $G_2$; these can be estimated in normal circumstances, as will be discussed in more detail below. Furthermore, the variance of $p$ at an infinity of learning iterations can be made arbitrarily small by assuming that learning is sufficiently slow, i.e.~if the learning rate $\gamma$ has a small value. Thus, in the limit of long learning periods and slow learning rates, a population of learners are all expected to behave the same, converging to similar values of $p$ at the end of learning. In effect, the behaviour of an entire generation of learners (subject to the same learning environment) can be condensed into one number, namely the limit of the expectation \eqref{eq:LRP-asymptote}.

The empirical justification for this procedure stems from the fact that language learning typically involves very numerous learning iterations. It has been estimated that people speak around 16,000 words a day on average \autocite{MehEtAl2007}. By conversational symmetry, we would expect people, and language learners in particular, to hear a similar number of words every day. Language acquisition ordinarily takes place over several years; translating the 16,000-words-a-day estimate into the number of tokens relevant for the acquisition of any reasonably frequently occurring linguistic variable then yields an estimate in the millions.\footnote{Note that the two empirical variables considered in this paper, verbal inflection (or lack thereof) and subject expression (or lack thereof), will figure in the vast majority of utterances heard by the learner.} Although it is possible that adult L2 learners are exposed to less input than L1 learners, the high order of magnitude of this estimate suggests that it is legitimate to focus on the learner's asymptotic behaviour at large learning iterations, instead of trying to characterize the complex stochastic dynamics of the entire learning trajectory. First pioneered by \textcite{Yan2000}, this strategy has been used by numerous authors to derive population-level or diachronic consequences of sequential generations of such learners \autocite{Yan2002,HeyWal2013,IngLegYan2013,Dan2017,KauWal2018,Kau2019,SimEtAl2019}.

The majority of previous studies applying the VL framework have applied it to L1 learning. The idea of grammar competition has, however, been used in a handful of L2 acquisition studies \parencite{ZobLic2005,Ran2014,Ran2022}, and it is possible also to harness the mathematics of the linear reward--penalty algorithm for the purposes of modelling L2 trajectories formally. To accommodate the existence of L2-difficult features, we include a bias against successful L2 acquisition of one of the competing grammars. This may be done by replacing the operators $f$ and $g$ with the following pair of operators:
\begin{equation}
  \label{eq:L2LRP}
    \begin{aligned}
      f'(p) &= p + \gamma (1-p) - \delta p, \\
      g'(p) &= p - \gamma p - \delta p,
    \end{aligned}
\end{equation}
where $\delta$ is a small positive number that quantifies the difficulty faced by L2 learners in acquiring $G_1$.\footnote{To guarantee that $p$ always remains in the interval $[0,1]$ and thus a probability, one must require that $0 < \delta < 1 - \gamma$. See the Appendix for details.} It is important to note that with this definition, the difficulty is endemic to L2 learning, in the sense that it is independent of the learning environment's responses: no matter how much the environment rewards the use of $G_1$, this option will always suffer some amount of penalty, modulated by the magnitude of $\delta$.

In the Appendix, it is shown that an equivalent asymptotic result holds for this L2 extension of the original linear reward--penalty learning scheme. As learning iteration tends to infinity, the expected value of $p$ with operators $f'$ and $g'$ tends to
\begin{equation}
  \label{eq:L2LRP-asymptote}
  \langle p \rangle'_{\infty} = \frac{\pi_2}{\pi_1 + \pi_2 + d},
\end{equation}
where $d = \delta/\gamma$ represents the L2-difficulty of $G_1$ scaled by the learning rate parameter $\gamma$. Moreover, the variance can again be made arbitrarily small by taking a small enough value of $\gamma$ (and modifying $\delta$ so as to keep the ratio $d$ constant). In other words, a population of L2 learners employing operators $f'$ and $g'$ can be described by a single number at the asymptote just as a population of L1 learners employing the standard linear reward--penalty scheme can.

\section{Population dynamics}
\label{sec:population-dynamics}

\subsection{Motivation}

It is conceptually useful to think of language change as an intertwined process of innovation and propagation. In the present model, L2 learners who acquire a lower probability of employing the L2-difficult grammar (as compared to L1 learners) are innovators. Whether and how these innovations propagate across the population depends, roughly, on how prevalent the L2 learner population is in the entire speech community. In general, interactions between individuals in a speech community are the result of a complex combination of factors involving properties of social networks, geographical distance and interaction frequency, to name but a few. Such factors can in principle be encapsulated in stochastic models of social dynamics inspired by techniques borrowed from statistical physics \autocite{Hel2010}; however, inclusion of too many factors usually renders such models analytically intractable. Alternatively, when populations are large, stochastic fluctuations normally average out, and the system can be reduced to a simpler deterministic model, with correspondingly simpler analysis. This is the approach adopted here, developed in detail in Section \ref{sec:deterministic-limit}. Section \ref{sec:simulations} provides simulation-based support for the deterministic approximation.

\subsection{Deterministic approximation}
\label{sec:deterministic-limit}

Setting aside the complications arising from the full stochastic complexity of learning trajectories and interaction patterns, let us now focus on a simplified model in which learners behave according to the learning-theoretic expectations \eqref{eq:LRP-asymptote} and \eqref{eq:L2LRP-asymptote} and in which population size is so large that demographic noise cancels out. Technically, we consider an infinite well-mixed population of individuals in which the fraction of L2 speakers is $\sigma$ and the fraction of L1 speakers is $1 - \sigma$. Let $p$ and $q$ denote the probabilities of grammar $G_1$ in the L1 and L2 speaker populations, respectively (so that the probabilities of $G_2$ are $1-p$ and $1-q$). Following \textcite{Yan2000}, I assume a fraction $\alpha_1$ of the output of grammar $G_1$ to be incompatible with $G_2$, and similarly a fraction $\alpha_2$ of the output of $G_2$ to be incompatible with $G_1$. The parameters $\alpha_1$ and $\alpha_2$ will be called the \emph{(grammatical) advantages}\footnote{Nomenclature is sometimes confusing in the literature, with some authors using the term ``fitness'' for these quantities and reserving the term ``advantage'' for derived notions such as difference in fitnesses. I here follow the original terminology put forward in \textcite{Yan2000}. Ultimately (see below), the relative fitness of the competing grammatical options will be measured using the ratio $\alpha = \alpha_1/\alpha_2$, with greater (smaller) than unity values signifying that $G_1$ ($G_2$) has more advantage than its competitor.} of $G_1$ and $G_2$ in what follows; numerical estimates will be given in Section \ref{sec:application}. Assuming for simplicity that learners sample linguistic input from their environments at random, the penalty probabilities of the two grammars can then be expressed in the following simple forms:
\begin{equation}
  \label{eq:penalties}
    \begin{aligned}
      \pi_{1} &= (1 - \sigma) \alpha_2 (1-p) + \sigma \alpha_2 (1-q), \\
      \pi_{2} &= (1 - \sigma) \alpha_1 p + \sigma \alpha_1 q.
    \end{aligned}
\end{equation}
To unpack this, note that the first term on the right hand side of the first equation, for instance, represents the event of the learner interacting with an L1 speaker ($1-\sigma$) who employs grammar $G_2$ ($1-p$) and utters a sentence which falls among those not compatible with $G_1$ ($\alpha_2$). Similarly, the second term on the right hand side of the second equation represents the case of the learner interacting with an L2 speaker ($\sigma$) who employs grammar $G_1$ ($q$) and utters a sentence not compatible with $G_2$ ($\alpha_1$), and similarly for the remaining two terms.

Following \textcite{Yan2000}, we can now assume that the input to the language acquisition process of the ($n+1$)th generation of speakers is constituted by the linguistic output of the $n$th generation. Assuming learners reach the learning-theoretic asymptote as argued in Section \ref{sec:learning}, this implies setting
\begin{equation}
  \begin{aligned}
    p_{n+1} &= \langle p_n \rangle_{\infty}, \\
    q_{n+1} &= \langle q_n \rangle_{\infty}',
  \end{aligned}
\end{equation}
with the expectations given by \eqref{eq:LRP-asymptote} and \eqref{eq:L2LRP-asymptote}. Expanding these with the help of the penalties \eqref{eq:penalties}, we arrive at the pair of equations
\begin{equation}
  \label{eq:the-system-before-alpha}
  \begin{aligned}
    p_{n+1} &= \frac{(1 - \sigma) \alpha_1 p_n + \sigma \alpha_1 q_n}{(1 - \sigma) \alpha_2 (1 - p_n) + \sigma \alpha_2 (1 - q_n) + (1 - \sigma) \alpha_1 p_n + \sigma \alpha_1 q_n}, \\
    q_{n+1} &= \frac{(1 - \sigma) \alpha_1 p_n + \sigma \alpha_1 q_n}{(1 - \sigma) \alpha_2 (1 - p_n) + \sigma \alpha_2 (1 - q_n) + (1 - \sigma) \alpha_1 p_n + \sigma \alpha_1 q_n + d}. \\
  \end{aligned}
\end{equation}
To reduce the number of parameters in these equations, it is useful to divide both the numerators and the denominators on the right hand side by $\alpha_2$ (on the assumption that $\alpha_2 \neq 0$). This condenses the relation of the two advantage parameters $\alpha_1$ and $\alpha_2$ into a single number, the advantage ratio $\alpha = \alpha_1 / \alpha_2$:
\begin{equation}
  \label{eq:the-system}
  \begin{aligned}
    p_{n+1} &= \frac{(1 - \sigma) \alpha p_n + \sigma \alpha q_n}{(1 - \sigma) (1 - p_n) + \sigma (1 - q_n) + (1 - \sigma) \alpha p_n + \sigma \alpha q_n}, \\
    q_{n+1} &= \frac{(1 - \sigma) \alpha p_n + \sigma \alpha q_n}{(1 - \sigma) (1 - p_n) + \sigma (1 - q_n) + (1 - \sigma) \alpha p_n + \sigma \alpha q_n + D}. \\
  \end{aligned}
\end{equation}
The quantity $D = d/\alpha_2$ represents the relative L2-difficulty of grammar $G_1$ scaled by the advantage of grammar $G_2$.

To recap, we have assumed grammatical competition between two options $G_1$ and $G_2$, the first of which is assumed to incur an amount of L2-difficulty, quantified by the ratio $d = \delta/\gamma$ of raw L2-difficulty $\delta$ to underlying learning rate $\gamma$. The two learning algorithms of Section \ref{sec:learning} lead to a mixed population of L1 and L2 speakers whose dynamics are described by the pair of equations \eqref{eq:the-system}, on the assumption that speakers mix randomly. The dynamics depend on three model parameters:
\begin{itemize}
  \item $\alpha = \alpha_1/\alpha_2$ controls the ratio of grammatical advantages, indicating how much (plain) advantage the L2-difficult grammar $G_1$ has in relation to grammar $G_2$;
  \item $D = d/\alpha_2$ represents the L2-difficulty of grammar $G_1$ in relation to the grammatical advantage of $G_2$;
  \item $\sigma$ gives the proportion of L2 speakers in the population.
\end{itemize}
Intuitively, one expects the (diachronic) loss of the L2-difficult grammar $G_1$ to be more likely for higher values of $D$ and $\sigma$, as well as for lower values of $\alpha$. These expectations are borne out by exact mathematical analysis, as will be described next.

In \eqref{eq:the-system}, the variables $p$ and $q$ represent the relative frequency of the L2-difficult grammar $G_1$ in the L1 and L2 speaker populations, respectively. We would now like to understand the range of behaviours this system is capable of. Each pair of values $(p,q)$ with $0 \leq p,q \leq 1$ is a possible population state, or in other words, the state space of the system consists of the unit square $[0,1] \times [0,1] = [0,1]^2$. Of particular interest is the eventual fate of the population under the above modelling assumptions. In general, some points $(p,q)$ of the state space may be expected to be attractors, drawing the population state to themselves over time, while other points may repel the population state. These configurations of the state space correspond to different empirical outcomes, as will be demonstrated next.

In the Appendix, it is proved that the system \eqref{eq:the-system} can display one of two eventual outcomes for any fixed combination of model parameters $\alpha$, $D$ and $\sigma$. These are:
\begin{enumerate}
  \item[I.] The L2-difficult grammar $G_1$ is used with some probability in one or both speaker groups. Mathematically, the system has two equilibria, the origin $(p,q) = (0,0)$ and a further, non-origin state $(p^*, q^*) \neq (0,0)$. The former, however, is unstable, while the latter is asymptotically stable (a sink). The population will be attracted to $(p^*, q^*)$ over time, with the consequence that the L2-difficult feature is retained in the language---with frequency $p^*$ in the L1 speaker population, and with frequency $q^*$ in the L2 speaker population.
  \item[II.] The L2-difficult grammar $G_1$ vanishes from \emph{both} groups. Mathematically, only one equilibrium exists, namely the origin $(p,q) = (0,0)$, which is asymptotically stable. Consequently both the L1 and L2 populations eventually speak $G_2$ exclusively.
\end{enumerate}
Of these outcomes, the latter corresponds to the proposed mechanism of L2 mutation followed by L1 nativization. The passage from phase I to phase II is governed by a complex interaction of the three parameters, one that however can be solved analytically. First, if $\alpha < 1$, so that grammar $G_1$ has less advantage to begin with, phase II is predicted. This is unsurprising: if both grammatical fitness in the VL sense and L2-difficulty in the Trudgillian sense conspire against a grammatical option, no force exists to sustain it, and the option is predicted to disappear.\footnote{The case $\alpha < 1$ is also empirically uninteresting in most situations, as it begs the question how grammar $G_1$ could have established itself in the L1 population in the first place, if it has less grammatical advantage than its competitor $G_2$.} The same outcome holds for the special case $\alpha = 1$, in which the grammatical advantages are equal.

For $\alpha > 1$, the situation is more complicated. If $\alpha \geq D + 2$, phase I is predicted for any combination of parameter values. In this case, the pure grammatical advantage enjoyed by the L2-difficult grammar $G_1$ is so high that no amount of L2 learning can suppress it entirely---the L1 speaker population will continue to speak $G_1$ at some non-zero frequency. For $1 < \alpha < D + 2$, however, the proportion of L2 speakers $\sigma$ acts as a bifurcation parameter. For values of $\sigma$ below a critical threshold
\begin{equation}
  \label{eq:sigma-crit}
  \sigma_{\mathrm{crit}} = \frac{(\alpha - 1)(D + 1)}{\alpha D},
\end{equation}
phase I is predicted. In this case, the proportion of L2 learners in the population is not high enough to completely suppress $G_1$. For values of $\sigma$ exceeding this threshold, however, phase II is predicted: the non-origin equilibrium $(p^*,q^*)$ vanishes and the population converges to $(p,q)=(0,0)$, with both L1 and L2 speakers now employing grammar $G_2$ exclusively. In other words, as the proportion of L2 learners in the population grows, the speech community experiences a phase transition from phase I to phase II (Figure \ref{fig:bifurcation} and Table \ref{tab:fates}).\footnote{A reviewer asks why the proportion of L2 speakers $\sigma$ is chosen as the critical parameter, rather than $\alpha$ or $D$. Indeed, from a mathematical point of view the choice is arbitrary: as long as the inequality $\sigma > \sigma_{\mathrm{crit}}$ is satisfied, loss of the L2-difficult feature is predicted. So, for instance, we may hold $\alpha$ and $D$ fixed and vary $\sigma$ to cross from one phase to the other, but we may equally well hold $\alpha$ and $\sigma$ fixed and vary $D$ to undergo that phase transition. However, from a substantive point of view, the proportion of L2 speakers $\sigma$ is the only parameter that routinely changes its value---the advantage ratio $\alpha$ stays fixed as long as no other grammatical changes occur in the language, and $\delta$ (and, by extension, $D$ whenever $\alpha_2$ is fixed) is presumably constant as it reflects a universal psychological bias. Hence it is natural to treat $\sigma$ as the critical control parameter.}

\begin{figure}
  \centering
  \includegraphics[width=1.0\textwidth]{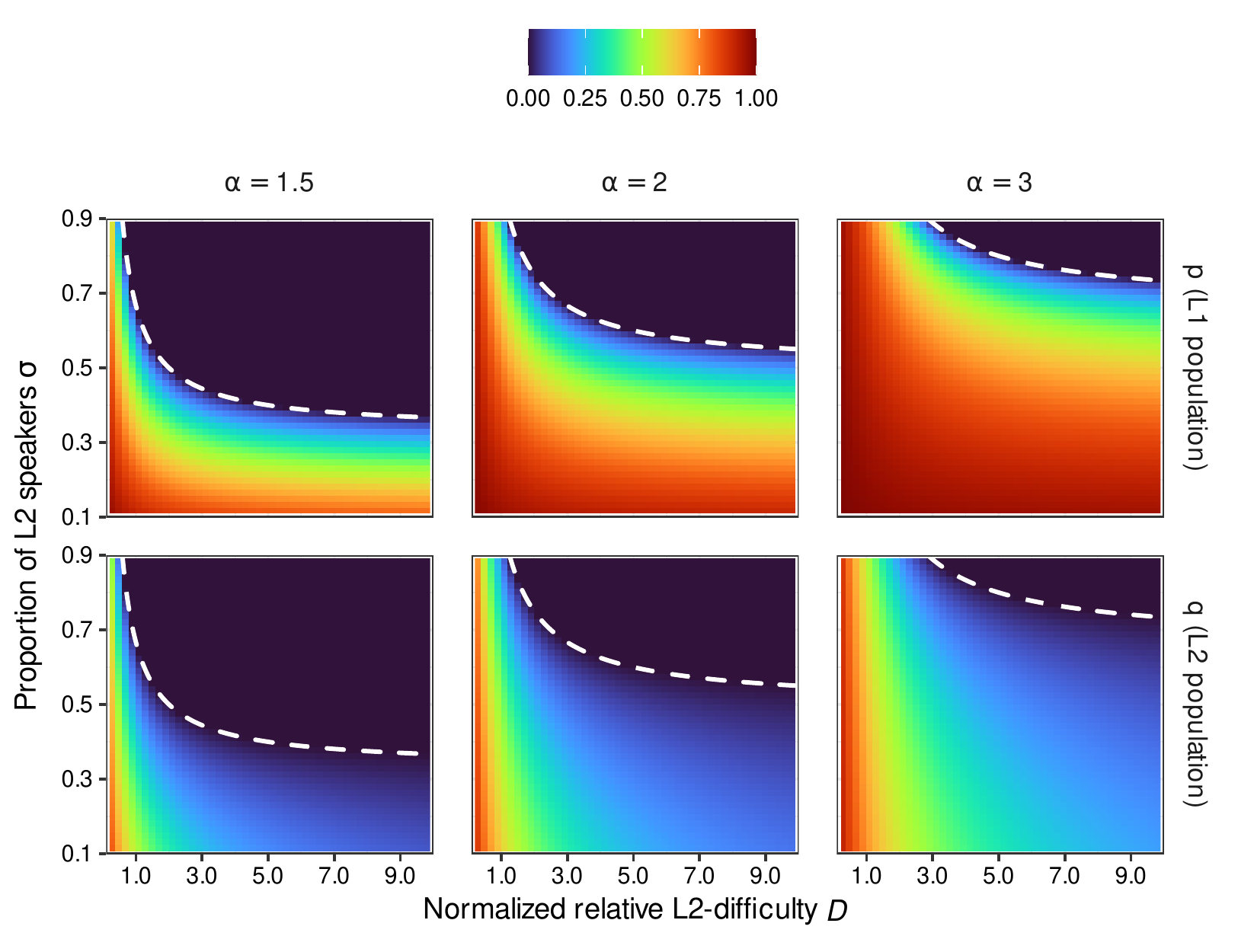}
  \caption{Orbit diagram of the system showing the values of $p$ (top row) and $q$ (bottom row) at the stable equilibrium, for various values of advantage ratio $\alpha$ (columns), L2-difficulty $D$ (horizontal axis) and proportion of L2 speakers $\sigma$ (vertical axis). The dashed line supplies the bifurcation boundary $\sigma_{\textnormal{crit}}$: the L2-difficult grammar $G_1$ is extinct in each speaker population above this line, but coexists with grammar $G_2$ below it.}\label{fig:bifurcation}
\end{figure}

\begin{table}
  \centering
  \caption{Eventual outcome of contact situation, for different combinations of advantage ratio $\alpha = \alpha_1/\alpha_2$, normalized L2-difficulty $D = d/\alpha_2$ and proportion of L2 learners $\sigma$. The critical value of the latter is given in \eqref{eq:sigma-crit}.}\label{tab:fates}
  \begin{tabular}{lr}
    Condition & Fate of L2-difficult grammar $G_1$ \\
    \hline
    $0 < \alpha \leq 1$ & lost \\
    $1 < \alpha < D+2$ and $\sigma > \sigma_{\mathrm{crit}}$ & lost \\
    $1 < \alpha < D+2$ and $\sigma < \sigma_{\mathrm{crit}}$ & retained \\
    $\alpha \geq D+2$ & retained \\
    \hline
  \end{tabular}
\end{table}

\subsection{Finite simulations}
\label{sec:simulations}

The above deterministic model turns on the assumption that learners are well-described by the learning-theoretic asymptotic expectations \eqref{eq:LRP-asymptote} and \eqref{eq:L2LRP-asymptote}, as well as on the assumption that interactions between speakers are sufficiently random so that stochastic fluctuations in the population average out. These idealizing assumptions were made deliberately, to unleash the greater explanatory power of analytically soluble models, compared to simulations \parencite[see][4--11]{McElrBoy2007}. Nevertheless, the assumptions can be debated, and the question of how a finite, stochastic system would behave is certainly an interesting one. Although it is impossible to explore the entirety of the full stochastic model's parameter space by way of simulations in this paper, I here report the outcome of proof-of-concept simulations which lend tentative support to the deterministic approximation analysed in Section \ref{sec:deterministic-limit}.

The choice to abstract away from the individual-level stochastic dynamics of language acquisition was defended by way of an argument from the typically slow pace of language acquisition in Section \ref{sec:learning}; this entails assuming that learners receive a large amount of input during their learning period and also make only small adjustments at each presentation of input token, so that the learning rate parameter $\gamma$ has a small value. In the absence of direct empirical estimates of $\gamma$, it is legitimate to worry about the effects that high learning rates may have on the ensuing dynamics. To explore this, ten Variational Learners were simulated for varying values of $\gamma$, exposed to a learning environment in which the relative frequency of grammar $G_1$ was $0.5$ and the advantage parameters had the values $\alpha_1 = 0.25$ and $\alpha_2 = 0.2$. It was moreover assumed that grammar $G_1$ incurs an L2-difficulty of $d = 2$. Of the ten learners, five were randomly assigned to be L1 learners; the other five were L2 learners subject to the L2-difficulty of $G_1$. Each learner received 100,000 input tokens over their learning period and started from a randomly drawn initial value of $p$ or $q$.

\begin{figure}
  \centering
  \includegraphics[width=1.0\textwidth]{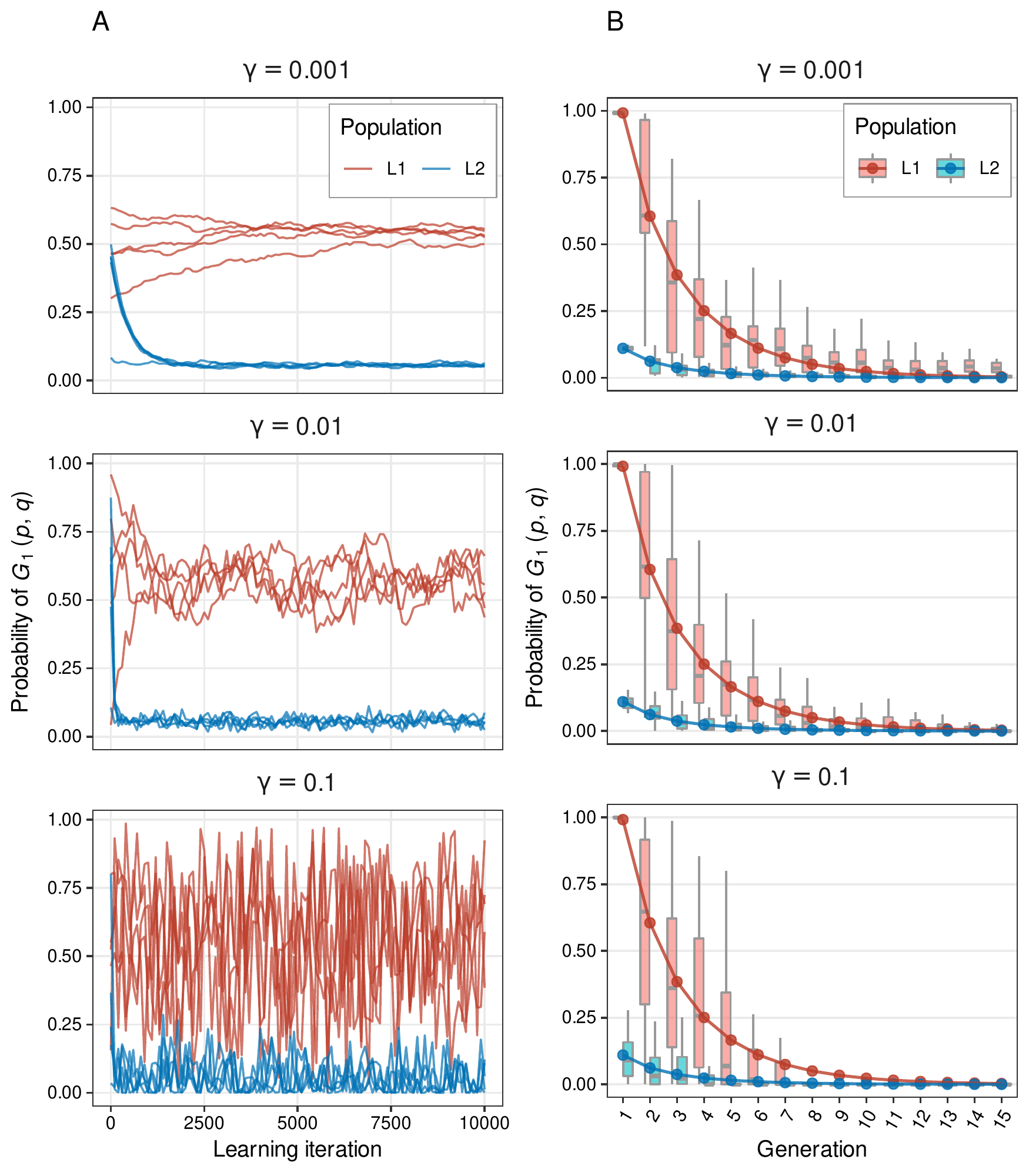}
  \caption{Finite simulations of the model. (A) Learning trajectories of 10 individual learners, half of them L1 and the other half L2 learners, for different learning rates $\gamma$. First 10,000 iterations shown only, sampled every 100 iterations for increased clarity; behaviour after this initial transient period is identical. (B) Diachronic trajectory of 15 generations of learners, 100 learners in each generation (half L1, half L2). The boxplots characterize the simulated data, i.e.~the distribution of $p$ (or $q$) across all learners at the end of their learning period. The solid curves give the deterministic prediction using equation \eqref{eq:the-system}. For simulation parameters, see text.}\label{fig:simulations}
\end{figure}

With these assumptions, the learning-theoretic expectations \eqref{eq:LRP-asymptote} and \eqref{eq:L2LRP-asymptote} predict an average probability of $p = 0.56$ for grammar $G_1$ in the L1 learner population, and an average of $q = 0.06$ in the L2 learner population. Figure \ref{fig:simulations}A presents the simulation results. As expected, the predicted averages describe the population averages. On the other hand, for higher learning rates (larger values of $\gamma$), a larger residual variance about this expectation is observed across the population. It turns out, however, that this residual inter-learner variance, even when large, need not affect inter-generational, diachronic developments. Consider Figure \ref{fig:simulations}B, which tracks the inter-generational trajectory of $p$ and $q$ across $15$ generations of learners, each generation consisting of 100 learners (half L1, half L2), starting from $p = q = 0.99$ in generation $0$ and subject to the same parameters as above. To channel input from generation $n$ to generation $n+1$ in a reasonably realistic way, each learner in generation $n+1$ was randomly drawn two ``parents'' from generation $n$; each learner only received input from its parents. The consequence is that, particularly in the intermediate stages of the diachronic development, and particularly when the learning rate parameter $\gamma$ has a high value, there is much variation across speakers in the community. However, since the ultimate force driving the evolution of $p$ and $q$ is deterministic in nature (a difference in the advantages of the competing grammars, combined with L2-difficulty of one of the options), this stochastic noise is not sufficient to disturb the development. With the above parameters, equation \eqref{eq:sigma-crit} implies a critical L2 proportion threshold of $\sigma_{\textnormal{crit}} = 0.3$; since the actual proportion of L2 learners in the simulation is $0.5$, we expect grammar $G_1$ to be driven to extinction over inter-generational time, i.e.~the values of $p$ and $q$ to converge toward zero. This is exactly what is observed, even for noisy (high) values of the learning rate parameter. Moreover, the deterministic trajectories computed from \eqref{eq:the-system}, shown in Figure \ref{fig:simulations}B as the connected curves, are broadly consistent with the simulated data in each generation. In fact, with higher learning rates, the predicted equilibrium $(p,q) = (0,0)$ is attained quicker.

\section{Application}
\label{sec:application}

\subsection{Interim summary}

To summarize the results of the foregoing sections, our mixed population of L1 and L2 speakers is capable of two qualitatively different kinds of long-term behaviour. The parameter $\alpha = \alpha_1 / \alpha_2$ expresses the ratio of the advantages of the two competing grammars. If $\alpha < 1$, the state $(p,q) = (0,0)$, in which both the L1 and L2 speaker populations use $G_2$ to the complete exclusion of the L2-difficult grammar $G_1$, is a stable equilibrium. If $\alpha > 1$, this state is either stable or unstable depending on whether $\sigma$, the proportion of L2 speakers in the population, exceeds a critical threshold $\sigma_{\textnormal{crit}}$. The value of $\sigma_{\textnormal{crit}}$, in turn, depends on the magnitudes of $\alpha$ and $D = d/\alpha_2$, the latter expressing the relative L2-difficulty of grammar $G_1$, scaled by the grammatical advantage of $G_2$. In this scenario, two evolutionary outcomes are possible: either total extinction of the L2-difficult grammar $G_1$ (when $\sigma > \sigma_{\textnormal{crit}}$), or stable variation between $G_1$ and $G_2$ (when $\sigma < \sigma_{\textnormal{crit}}$). 

The three model parameters $\alpha$, $\sigma$ and $D$ are all estimable in principle from empirical data: $\alpha$ from the frequencies of occurrence of different types of linguistic constructions, $\sigma$ from population censuses, and $D$ from L2 learning data. Suitable data of the last kind are presently lacking,\footnote{What is required is access to individual-level longitudinal production data from L2 learners over a substantial stretch of their learning trajectory, against which theoretically predicted learning curves could be fit. If learning curves with $\delta > 0$ fit such data better than curves with $\delta = 0$, we would have positive evidence for the L2-difficulty.} but I will next attempt calibration of the first two parameters in the specific cases of Afrikaans deflexion and the erosion of null subjects in Afro-Peruvian Spanish. Here it should be borne in mind that all parameter estimates can be only approximate: some degree of uncertainty necessarily pertains to corpus estimates of grammatical advantages, and historical population data are wrought with problems well known to historians and demographers. I will therefore offer the parameter estimates as an approximate starting point, with the proviso that they come with an unknown degree of statistical uncertainty. This has important consequences on what we take the very goal of modelling to be. Even though exact \emph{quantitative} statements about the equilibrium state of the linguistic system may be out of reach, it is still possible to infer something about the \emph{qualitative} outcome of the contact situation in each empirical test case.

\subsection{Calibrating grammatical advantages}
\label{sec:calibrating-advantages}

In the case of Afrikaans verbal deflexion, the relevant competition is between a grammar that has person and number distinctions in the verbal paradigm ($G_1$, corresponding to Dutch) and a grammar that doesn't ($G_2$, corresponding to Afrikaans). Parsing failure occurs when the learner--listener attaches the wrong interpretation (wrong person or number) to the surface form uttered by the speaker. Since Dutch is a non-NS language, person and number can always be inferred from the nominal domain, even if verbal inflection is eroded. On the face of it, this implies that the raw grammatical advantages of the two grammars are equal, so that $\alpha = 1$. In practice, it may be argued that factors such as channel noise may lend the grammar with verbal inflection ($G_1$, i.e.~Dutch) a slight advantage over the inflectionless grammar; on the other hand, redundant agreement has been suggested to present difficulty for L2 learners \autocite{Kus2008}. However, the magnitudes of these effects are difficult to estimate outside the context of a strict laboratory study. In what follows, I will assume that such effects are negligible at the population level, and thus proceed with the maximum-parsimony assumption that $\alpha = 1$, i.e.~that neither grammar is more advantageous than its competitor.

The case of null subjects in Afro-Peruvian Spanish is markedly different. The NS grammar ($G_1$) is incompatible with overt expletive subjects \eqref{ex:el-llueve-2}, while the non-NS grammar ($G_2$) is incompatible with null thematic subjects \eqref{ex:hablo-espanol-2}:
\begin{examples}
\item\label{ex:el-llueve-2} 
  Overt expletive subject pronouns (* for $G_1$)
  \digloss{Ello llueve.}{it rains}{It's raining.}
\item\label{ex:hablo-espanol-2}
  Null thematic subject pronouns (* for $G_2$)
  \digloss{$\nothing$ Hablo español.}{{} speak Spanish}{I speak Spanish.}
\end{examples}
As one might expect, thematic subjects far outnumber expletive subjects in discourse. Drawing on data on their frequencies of occurrence from multiple sources, \textcite[296, 299]{SimEtAl2019} estimate the grammatical advantage of a NS grammar to be about $\alpha_1 = 0.7$ and that of the corresponding non-NS grammar to be only about $\alpha_2 = 0.05$; that is to say, a difference of over an order of magnitude, implying an advantage ratio of $\alpha = 0.7/0.05 = 14$ in favour of the NS grammar.\footnote{If learners of a NS grammar are learning a syntactic parameter that controls a cluster of surface features, as classical formulations of the null subject parameter expect (\citeauthor{RobHol2010} \citeyear{RobHol2010}; but see \citeauthor{SimEtAl2019} \citeyear{SimEtAl2019} for evidence to the contrary), then there will be further cues in the learner's input that either reward or penalize the two competing settings (such as rich vs.~poor agreement inflection, or the possibility vs.~impossibility of free inversion). Such complications fall outside the scope of the present paper but should be explored in future work.}

This difference between the two case studies is important. In the case of Afrikaans, the fact that the two grammars are equiadvantageous means that, in the absence of any L2 learning, the population of speakers is stable (though not asymptotically) at any value of $p$: if $\sigma = 0$ and $\alpha = 1$, then the value of $p$ is always at equilibrium.\footnote{To see this, we take the first equation from \eqref{eq:the-system} and set $\sigma = 0$ and $\alpha = 1$; the resulting expression $p_{n+1} = p_n$ means that the value of $p$ will not change from generation to generation.} Thus the only forces that could shift the state of the speech community (again, assuming the absence of L2 learners in the population) would have to be stochastic in nature. Once L2 learners facing an L2-difficulty with one of the grammars are introduced into the population, the equilibrium will shift: any amount of L2 learning implies that the origin $(p,q) = (0,0)$ becomes the system's only attractor (Table \ref{tab:fates}). Were those L2 learners to be removed from the population at a later stage, the speech community would then again assume its previous quasistable nature and come to rest at whichever value of $p$ the L2 learning situation brought it to (this will correspond to $p \approx 0$ if enough time has elapsed; see Section \ref{sec:predictions-afrikaans}).

The case of null subjects is radically different. Here the original, L2-difficult grammar $G_1$ is much more advantageous than its competitor, $G_2$. Not only does this mean that the proportion of L2 learners in the population would have to be relatively high for $G_1$ to be completely replaced by $G_2$; the long-term prediction is also that if those L2 learners were to be removed, the population would drift back to the equilibrium $p=1$ whose stability is guaranteed by the asymmetry in grammatical advantages in the absence of any L2 learning.

\subsection{Calibrating demographics}
\label{sec:calibrating-demographics}

\begin{table}[t]
  \caption{Demographics of Dutch Cape Colony, 1670--1820 \autocite[360]{GilElph1979}, together with estimated range for the fraction of people speaking Dutch as L2, $\sigma$ (see text). Column for Khoekhoe includes people of mixed background; these data are not available prior to 1798.}\label{tab:afrikaans-demographics}
  \centering
  \begin{tabular}{|c|c|c|c|c|c|}
    \hline
    Year & Europeans & Free Blacks & Slaves & Khoekhoe & Prop.~L2 ($\sigma$) \\
    \hline
    1670 & 125 & 13 & 52 & --- & $[0.17, 0.34]$ \\
    \hline
    1690 & 788 & 48 & 381 & --- & $[0.18, 0.35]$ \\
    \hline
    1711 & 1,693 & 63 & 1,771 & --- & $[0.26, 0.52]$ \\
    \hline
    1730 & 2,540 & 221 & 4,037 & --- & $[0.31, 0.63]$ \\
    \hline
    1750 & 4,511 & 349 & 5,327 & --- & $[0.28, 0.56]$ \\
    \hline
    1770 & 7,736 & 352 & 8,220 & --- & $[0.26, 0.53]$ \\
    \hline
    1798 & c.~20,000 & c.~1,700 & 25,754 & 14,447 & $[0.34, 0.68]$ \\
    \hline
    1820 & 42,975 & 1,932 & 31,779 & 26,975 & $[0.29, 0.59]$ \\
    \hline
  \end{tabular}
\end{table}

Table \ref{tab:afrikaans-demographics}, from \textcite{GilElph1979}, gives an overview of the population of the Dutch Cape Colony from 1670 to 1820, the period relevant for the emergence of Afrikaans. Although it is evident from these figures that slaves, freed slaves and the indigenous peoples represented a significant fraction of the population at each stage, translating these data into figures of the likely proportion of L2 learners at the different time points is non-trivial. First, no data exist for the native population before 1798. Secondly, not all of the non-European population would be L2 learners: perhaps some of them would not have needed to learn Dutch, but even more importantly, at some point part of this population will have begun to acquire Dutch as an L1, or at least as a bilingual L2 in childhood. For these reasons, I have computed fairly wide interval estimates for $\sigma$, the proportion of (adult) L2 speakers in the colony, given in the rightmost column of Table \ref{tab:afrikaans-demographics}. The left endpoint of each interval represents the conservative estimate that half of the non-European population would have spoken Dutch as an L2; the right endpoint gives the absolute maximum, on the (quite certainly unrealistic) assumption that the entire non-European population spoke Dutch as an L2.

Population figures for Colonial Peru are hard to come by; however, useful data exists for Lima in the early period, specifically the first half of the 17th century. Interpretation of the demographics in Table \ref{tab:spanish-demographics}, from \textcite{Bow1974}, is complicated by the fact that the various sources from which the demographic data are drawn employ different classificatory principles, sometimes pooling Spaniards and people of mixed European--American background, or people of African and people of mixed African--European background, together. \textcite[93]{SesBook} argues that people of mixed background were most probably bilingual in Spanish from childhood (as one of their parents would have been Spanish). They should therefore be exempted from any estimates of the proportion of adult L2 learners in the population. The interval estimates given in Table \ref{tab:spanish-demographics} were computed with the assumption that only the Black and Indigenous groups would contribute adult L2 learners to the population; the left endpoints of these intervals again correspond to the conservative assumption that 50\% of these people spoke Spanish as L2, while the right endpoints correspond to the upper bound assumption that 100\% of them did.\footnote{To compute the estimates for the year 1600, for which Bowser's (\citeyear{Bow1974}) data pools people of African and mixed African--European heritage together, I have used the overall fraction of African to African--European for the remaining years ($0.92$). This yields an estimate of 6,091 African and 530 mixed African--European in the year 1600.} The resulting estimates of $\sigma$ suggest a steadily, if slowly growing proportion of L2 speakers in the population, with figures roughly similar to those found in the Dutch Cape Colony.

\begin{table}
  \caption{Demographics of Lima, 1600--1636 \autocite[339--341]{Bow1974}, together with estimated range for the fraction of people speaking Spanish as L2, $\sigma$ (see text). Mixed AmE $=$ mixed American--European background; Mixed AfE $=$ mixed African--European background.}\label{tab:spanish-demographics}
  \centering
  \begin{tabular}{|c|c|c|c|c|c|c|}
    \hline
    Year & Spanish & Mixed AmE & Black & Mixed AfE & Indigenous & Prop.~L2 ($\sigma$) \\
    \hline
    1600 & 7,193 & --- & \multicolumn{2}{c|}{6,621} & 438 & $[0.23, 0.46]$ \\
    \hline
    1614 & 11,867 & 192 & 10,386 & 744 & 1,978 & $[0.25, 0.49]$ \\
    \hline
    1619 & \multicolumn{2}{c|}{9,706} & 11,997 & 1,166 & 1,406 & $[0.28, 0.55]$ \\
    \hline
    1636 & 10,758 & 377 & 13,620 & 861 & 1,426 & $[0.28, 0.56]$ \\
    \hline
  \end{tabular}
\end{table}

I am not aware of useful demographic data for later stages of Colonial Peru; a census was carried out in 1725--1740, but it contains no information on slaves \autocite{Pea2001}. Useful information is available, however, on the population dynamics of the slave population on individual plantations or haciendas in the later period; while such data can tell us little about the relative proportion of L2 speakers in the population, they are useful in illuminating further developments following the initial stages of colonization. Drawing upon data in \textcite{Cus1980}, \textcite[107]{SesBook} reports an estimated yearly net growth of $1.4$ slaves per hacienda for the period 1710--1767; yet at the same time, the yearly natural growth (estimated from records of slaves' births and deaths, which were kept on haciendas run by Jesuits), is negative at $-2.7$ per hacienda. In other words, an average of $4.1$ slaves must have been imported yearly, per hacienda, in this period to match the net growth rate. It follows that, as late as the second half of the 18th century, there must have been a steady influx of new slaves, and hence of potential adult L2 learners of Spanish, into these regions.

\subsection{Predictions: Afrikaans}
\label{sec:predictions-afrikaans}

The main linguistic features of Afrikaans are estimated (based on extant written records) to have been in place by the end of the 18th century \autocite[83]{Rob2002}. Since the Dutch founded their permanent colony in 1652 \autocite{Gue1979}, this leaves around 150 years for the development of Afrikaans as a language separate from Dutch. Assuming one generation to correspond to roughly 30 years \autocite{TreVez2000}, the development is thus estimated to have occurred over about five generations of speakers. More precisely, five generations supplies an upper bound for the development; the linguistic change may have happened faster, too, but this is impossible to determine given the scant textual record in the earlier phases.

In the well-mixing, infinite-population setup adopted in Section \ref{sec:deterministic-limit}, the fact that the grammatical advantages are in balance in this case ($\alpha = 1$) means that, in theory, even one adult L2 learner would suffice to drive the L2-difficult grammar to extinction, given enough time, irrespective of the absolute size of the population. The crucial question then concerns the time scale of the development: for fixed grammatical advantage ratio $\alpha$ and L2-difficulty $d$, the proportion of L2 speakers $\sigma$ directly controls the time to extinction of $G_1$, i.e.~how long it takes for the population to converge to the attractor $(p,q) = (0,0)$ from some initial state $(p_0, q_0) = (1, q_0)$. (The initial probability in the L2 population, $q_0$, is of course an unknown.) Since the strength of L2-difficulty $d$ is also unknown, the best one can do is to compute these times-to-extinction for several combinations of the parameter values, thereby hopefully establishing at least a subset of the parameter space that predicts the empirically observed facts.

Iterating the pair of equations \eqref{eq:the-system} repeatedly, one can find the number of generations it takes for the population to converge to the attractor $(p, q) = (0,0)$ from various initial conditions $(p_0,q_0) = (1, q_0)$, for various selections of $d$ and for $\sigma = 0.2$ and $\sigma = 0.6$ (roughly corresponding to the endpoints of the interval estimates in Table \ref{tab:afrikaans-demographics}). Convergence to the attractor is here defined, somewhat arbitrarily, as the values of both $p$ and $q$ being below $0.001$ ($0.1\%$ use of $G_1$).

\begin{figure}[t]
  \centering
  \includegraphics[width=1.0\textwidth]{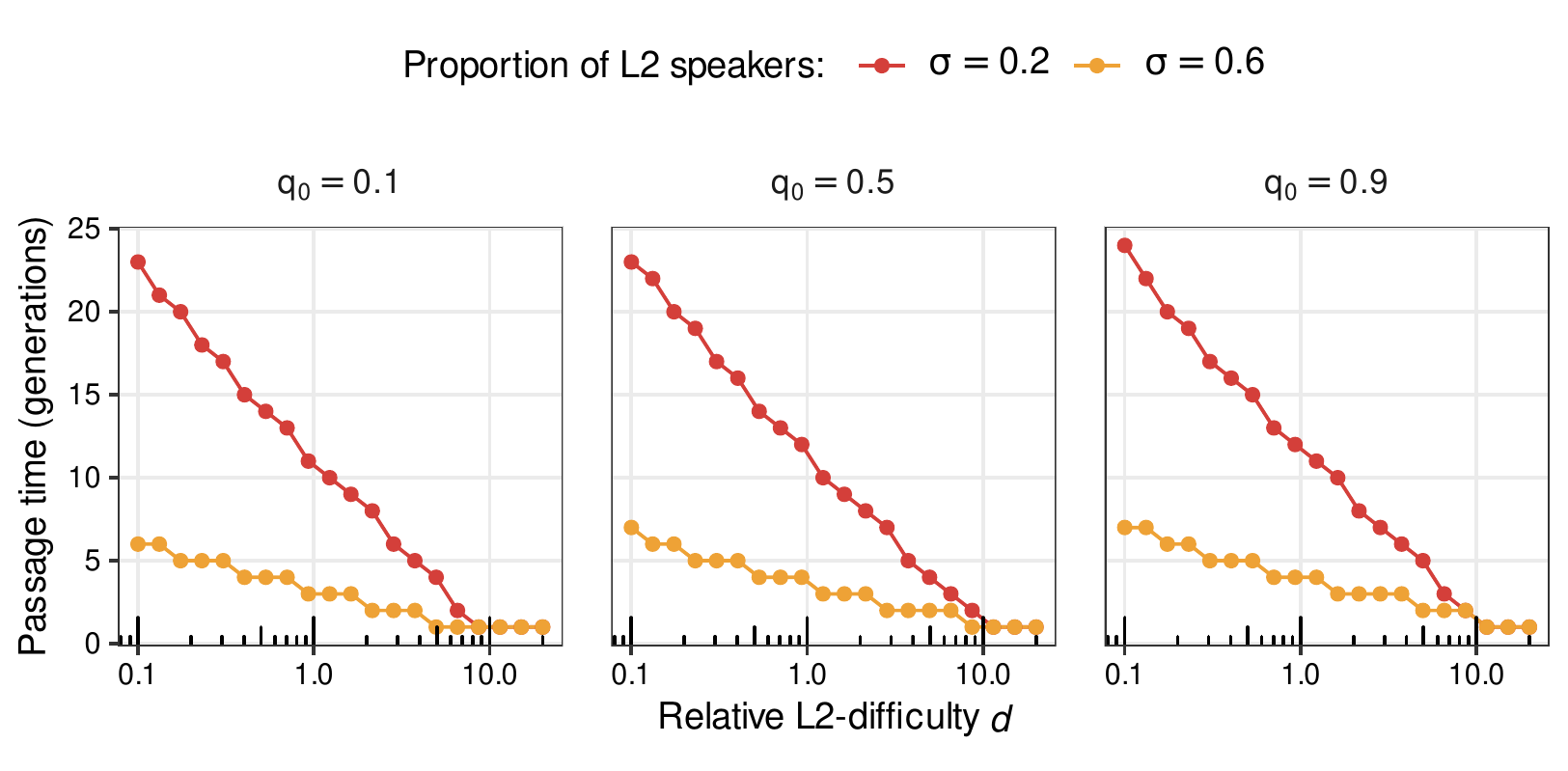}
  \caption{Time of convergence to attractor in the Afrikaans case study, for various values of L2-difficulty $d$, proportion of L2 speakers $\sigma$ and initial probability of L2-difficult grammar in the L2 speaker population $q_0$.}\label{fig:passage-times}
\end{figure}

Figure \ref{fig:passage-times} plots the passage times found using this procedure. From these results, it is clear that convergence to the attractor in 5 generations is possible. For proportions of L2 speakers on the order of $\sigma = 0.6$, a strength of L2-difficulty greater than about $d = 0.5$ is sufficient; for lower proportions on the order of $\sigma = 0.2$, L2-difficulties in excess of about $d = 5$ guarantee convergence in up to 5 generations. In future work, it would be important to estimate plausible values of $d$ using independent evidence, so that model fit can be evaluated in a more principled manner (subject to fewer researcher degrees of freedom). Having said that, the above demonstration shows that the model makes the development of Afrikaans \emph{possible} (even if not \emph{necessary}).

\subsection{Predictions: Afro-Peruvian Spanish}

In the case of Afro-Peruvian Spanish, the situation is radically different: in this case, it needs to be explained why the null subject grammar was never completely overthrown by the corresponding grammar without null subjects, which would be favoured by L2 learning. Recall that the critical value of the bifurcation parameter $\sigma$, the proportion of adult L2 learners in the population, was found in Section \ref{sec:population-dynamics} to be
\begin{equation}
  \sigma_{\textnormal{crit}} = \frac{(\alpha - 1)(D + 1)}{\alpha D},
\end{equation}
where $\alpha = \alpha_1/\alpha_2$ is the ratio of the advantages of the two grammars and $D = d / \alpha_2$ represents the relative L2-difficulty of $G_1$, scaled by the advantage of $G_2$. The lower bound of $\sigma_{\textnormal{crit}}$ occurs as $D \to \infty$, namely
\begin{equation}
  \sigma^{\infty}_{\textnormal{crit}} = \frac{\alpha - 1}{\alpha}.
\end{equation}
On the other hand, in Section \ref{sec:calibrating-advantages} the value $\alpha = 14$ was estimated in this particular case. This corresponds to a $\sigma^{\infty}_{\textnormal{crit}}$ of $13/14 \approx 0.93$. In other words, the proportion of L2 speakers in the population would need to be at least about $0.93$---and possibly higher, if $D$ turned out to have a small value---for the grammatical advantage enjoyed by the L2-difficult grammar $G_1$ to be overcome and for the origin $(p,q) = (0,0)$ to be an attractor. Since based on the available demographic data such a high proportion of L2 speakers never obtained in Colonial Peru (Section \ref{sec:calibrating-demographics}), we infer that null subjects were never about to be completely lost in the Afro-Peruvian variety of Spanish.

For any given non-zero $\sigma$, however, a stable, attracting state is implied in the interior of the state space (see again Section \ref{sec:deterministic-limit}). This would appear to correspond, qualitatively, to the linguistic classification of Afro-Peruvian Spanish as a mixed NS language (Section \ref{sec:afro-peruvian-spanish})---that is, as a variety sometimes employing, sometimes not employing, null subjects. It is worth pointing out, however, that as the proportion of L2 speakers decreases over time, this interior equilibrium tends to a stable rest point on the boundary of the state space, crucially with the property $p = 1$ (so that L1 speakers have reverted to using null subjects all the time). This may help to explain the reported instability of the mixed NS status of Afro-Peruvian Spanish---the observation that the younger speakers in these communities are turning towards the standard Spanish grammar, with full null subjects \autocite{Ses2014}. Although \textcite{Ses2014} attributes this development mostly to sociolinguistic causes---to the prestige enjoyed by standard coastal Peruvian Spanish---the above considerations suggest an alternative, or at least complementary analysis: the loss of overt subjects results from the fact that there are fewer L2 speakers of the variety in the speech community, leading to fewer constructions totally lacking null subjects in the input data based on which L1 learners acquire their variety.

\section{Conclusion and outlook}
\label{sec:concluding-remarks}

This paper has presented a mathematical model of contact-induced linguistic changes in which adult L2 learning plays a leading role. Fairly lenient assumptions about language learning were used to characterize the terminal state of learners exposed to a given learning environment; this terminal state was then used to derive a model of the population dynamics of a mixed population of L1 and L2 speakers. Focusing on a deterministic approximation of the full stochastic model allowed us to derive formal results about the equilibria and bifurcations of the system in relation to three main model parameters: $\alpha$, a measure of how much grammatical advantage the L2-difficult grammar $G_1$ has over its competitor $G_2$; $D$, a measure of the strength of the L2-difficulty of $G_1$; and $\sigma$, the proportion of L2 speakers present in the population. It was shown that the system has either one or two equilibria: either the point $(p,q) = (0,0)$, at which both the L1 and L2 populations speak $G_2$ to the complete exclusion of $G_1$, is asymptotically stable; or else this point is unstable and another asymptotically stable rest point exists in the state space. Passage from one of these phases to the other is controlled by the bifurcation parameter $\sigma$ with critical value
\begin{equation}
  \sigma_{\textnormal{crit}} = \frac{(\alpha - 1)(D + 1)}{\alpha D}
\end{equation}
whenever $\alpha < D + 2$. The origin $(p,q) = (0,0)$ is the only equilibrium when $\sigma > \sigma_{\textnormal{crit}}$.

Estimating empirical values of quantities such as grammatical advantages is now routine in variationist approaches to linguistic history \autocite{Yan2000,HeyWal2013,Dan2017,SimEtAl2019}. Estimation of demographic parameters such as the proportion of L2 speakers is also straightforward in principle (though by no means necessarily so in practice). By contrast, estimating the L2-difficulty suffered by a grammar is less trivial. Future work could potentially attempt to infer such parameters either from longitudinal data on L2 learning trajectories or from the terminal states attained by a larger sample of learners. In many cases, estimation of all three parameters may be unnecessary, however. As $D$ tends to infinity, $\sigma_{\textnormal{crit}}$ tends to $(\alpha - 1)/\alpha$. This supplies a lower bound for the critical value of the bifurcation parameter. Thus, if in a given case it is possible to show that the empirical value of $\sigma$ never exceeded this lower bound, then the prediction from the modelling is that the L2-difficult grammar should coexist with its competitor in the speech community. Conversely, as $D$ approaches $0$, $\sigma_{\textnormal{crit}}$ grows without bound. However, since $D = d/\alpha_2$ and $0 < \alpha_2 < 1$, it is often possible to argue that $D > 1$ as long as the unscaled L2-difficulty $d = \delta/\gamma$ can reasonably be assumed to be more than unity. This, then, provides an upper bound on the critical value of $\sigma$ (evaluated at $D=1$), namely $2(\alpha - 1)/\alpha$. If it is possible to show that the empirical value of $\sigma$ exceeds this upper bound, then the prediction from the modelling is that the L2-difficult grammar will be driven to complete extinction by the contact situation, given enough time.

The above results provide modelling-based evidence for the feasibility of the hypothesis of contact-induced change caused by imperfect L2 learning followed by L1 nativization \parencite{Wee1993,Tru2004,Tru2010,Tru2011}: it is possible for features to be lost from a language as a consequence of L2-difficulty. However, the results simultaneously suggest that extra-linguistic, population-dynamic parameters such as $\sigma$ may stand in complex relationships to structural, linguistic parameters such as $\alpha$ and parameters to do with the psychology of learning such as $d$. Thus it would be unrealistic to expect to find a simple and constant threshold of the proportion of L2 learners sufficient and necessary to cause change in the overall population. Nor is it realistic to expect all linguistic features or variables to respond identically to identical population-dynamic situations. The modelling results here suggest that such a view would be too simplistic. In fact, there may be empirical configurations of key system parameters which predict that adult L2 learning is insufficient to drive certain features to extinction, in certain population settings. This seems likely, for instance, for null subjects in cases such as that of Afro-Peruvian Spanish. On the other hand, in the case of variables where little to no advantage is enjoyed by either competing variant for purely linguistic reasons (such as verbal inflection in non-NS languages), those variables may be very sensitive to external factors such as L2 learning, with fairly low values of $\sigma$ sometimes sufficing to set the language on the course to lose the feature.

This paper has discussed one potential mechanism of contact-induced change. It assumes that L2 learners introduce a ``linguistic mutation'' to a speech community; this mutation then spreads as the primary linguistic data of L1 learners changes. Changes are cumulative and take place, typically, over extended inter-generational timescales. It is important to note that this is, however, only one possible mechanism that may generate the types of diachronic developments here studied. An alternative view holds that changes propagate as L1 speakers accommodate to the language use of L2 speakers in an act of foreigner-talk \parencite{Val1981,AtkSmiKir2018}. It is likely, in fact, that both mechanisms are at play at least to some extent in the real world. What is clear in any case is that the primary linguistic data of L1 learners will change regardless of the specific mechanism; it is this shift in input which, ultimately, secures propagation of innovations.

Future modelling work can, naturally, look into the effects of introducing more intricate mechanisms of skewed input. In a similar vein, it would be desirable in future work to model other aspects of L2 acquisition in more detail. In the present paper, universal L2-difficulty is the driving force of innovation. This difficulty is assumed to apply to all L2 learners in a similar way, regardless of the learner's L1. Although evidence for such a universal psychological difficulty was cited above, the existence of transfer effects in L2 acquisition is also well-documented in the literature \parencite{SchSpr1996}. In extensions of modelling work of this kind, such effects may be studied alongside universal biases. A related point is that, in the current framework, L1 and L2 acquisition are identical up to the effect of L2-difficulty; there is, for instance, no way in which L2 learners might sometimes outperform L1 learners (``positive transfer''). Such extensions could be considered in future research.

At the population-dynamic level, future work should explore in detail the effects of modelling the stochastic dynamics of interacting populations explicitly---beyond the proof-of-concept simulations reported in Section \ref{sec:simulations}. The great benefit of studying the deterministic limit of the model, in abstraction of stochastic fluctuations, is that strong analytical results become available. The bifurcation threshold identified in this paper is one such result: for any given combination of model parameters---proportion of L2 learners, extent of L2-difficulty, and advantage ratio---the model predicts either complete change or only partial change or no change at all. This increases the falsifiability of our theory, as all parameters can in principle be estimated from data and prediction compared against historical outcome. Should it turn out that the model makes the wrong predictions, some of its assumptions can be modified. For this reason, acquisition of further empirical test cases is an important desideratum. Although much current work into the role of demographic factors in language change is typological, synchronic and correlational in nature, diachronic data are essential in evaluating mechanistic models of change.

The results of this paper also add to a growing body of literature illustrating the utility of the Variational Learning model of linguistic variation and change. In the future, it will be important to move in the direction of more direct tests of key model ingredients, such as the learning rate parameter. Equally important is the balanced consideration and development of alternative formal models and, once the empirical predictions of each model have been figured out, the rigorous statistical comparison of competing models vis-à-vis empirical data.

\section*{Abbreviations}\label{abbrev}

AHLA = Afro-Hispanic Language of the Americas, L1 = first language, L2 = second language, NS = null subject, VL = variational learning

\section*{Data and code accessibility}

The programs required for replicating the simulations described in this paper can be obtained from \url{https://doi.org/10.5281/zenodo.7004002}.

\section*{Funding information}

The major part of the research here reported was funded by the European Research Council as part of project STARFISH (851423). The work was begun during the author's fellowship at the Zukunftskolleg of the University of Konstanz, funded by the Federal Ministry of Education and Research (BMBF) and the Baden-Württemberg Ministry of Science as part of the Excellence Strategy of the German Federal and State Governments. The Zukunftskolleg additionally provided funding for computing equipment. All this support is gratefully acknowledged.

\section*{Acknowledgements}

I am indebted to Fernanda Barrientos, Sarah Einhaus, Gemma McCarley, Raquel Montero, Molly Rolf, Joel Wallenberg and George Walkden for numerous discussions which have influenced my thinking on language learning, language change and population dynamics generally, as well as for specific comments pertaining to manuscript versions of the present work. I also wish to thank four anonymous reviewers whose comments helped to sharpen various aspects of the paper. Any remaining errors are, naturally, my sole responsibility.

\section*{Competing interests}

The author has no competing interests to declare.

\printbibliography

\eject
\begin{center}
  {\large\textsc{appendix}}
\end{center}

\appendix

\section{Learning: definitions}

Following \textcite{BusMos1955}, suppose the learner has $n$ possible actions (here, $n$ possible grammars) $G_1, \dots , G_n$ at their disposal and uses the $i$th action with probability $p_i$, so that the learner's knowledge is represented by the probability vector $\mathbf{p} = (p_1, \dots , p_n)$. Once the learner has chosen an action and acted on the environment, the latter responds with one of $m$ possible responses $R_1, \dots , R_m$. The probability of the $j$th response occurring, given that the learner chose the $i$th action, will be denoted $\omega_{ij}$. The set of these probabilities defines the (stationary random) learning environment; we require, of course, that $\sum_j \omega_{ij} = 1$ for all $i$. Suppose the learner chose $G_i$ and that the environment responded with $R_j$. Having observed this, the learner adjusts the vector $\mathbf{p}$ by applying an operator $f_{ij}$ which, in the general case, is only required to be some mapping $f_{ij} : \Delta^{n-1} \to \Delta^{n-1}$ from the simplex
\begin{equation}
  \Delta^{n-1} = \{\mathbf{p} = (p_1, \dots , p_n) \in [0,1]^n : \textstyle\sum_i p_i = 1 \}
\end{equation}
to itself. The process is then repeated: the next time, the learner chooses an action by drawing from the distribution $f_{ij}(\mathbf{p})$, and some $f_{k\ell}$ is applied to this vector to yield $f_{k\ell}(f_{ij}(\mathbf{p}))$, and so on.

Many of the formal properties of this general framework are understood in detail for various choices of operators $f_{ij}$, both for stationary (constant $\omega_{ij}$) and non-stationary (time-dependent $\omega_{ij}$) environments \parencite{BusMos1955,NarTha1989}. In what follows, I will utilize a special case in which the learner has two actions, the environment is stationary and has two responses, and the operators $f_{ij}$ are linear. The assumptions of stationarity and linearity facilitate characterization of the asymptotic distribution of a population of learners and turn out to give rise to tractable (although nonlinear) dynamics across generations of learners. The restriction to $n=2$ actions is not necessary, but is made to keep the presentation manageable and also in view of the empirical application, which concerns this special case. The two possible environmental responses will be interpreted as reward and punishment, in a sense to be made precise later.

Thus let $\mathbf{p} = (p_1, p_2) = (p, 1-p)$. From now on, I will interpret the variables in a linguistic setting and identify $p_1 = p$ with the probability of grammar $G_1$ and $p_2 = 1-p$ with the probability of grammar $G_2$. I assume in all that follows that grammar $G_1$ incurs some amount of L2-difficulty, while $G_2$ is not subject to such an inherent bias. In other words, adult L2 learners are expected to struggle more in acquiring $G_1$ than in acquiring $G_2$. 

Assuming the learner is making a binary choice between two grammars and that the environment signals two different responses only, four operators $f_{ij}$ need to be specified. Assuming further that these operators are linear in $\textbf{p} = (p_1, p_2) = (p, 1-p)$, we have
\begin{equation}
  f_{ij}(\textbf{p}) = M^{(ij)}\textbf{p}
  = \begin{bmatrix}
    m^{(ij)}_{11} & m^{(ij)}_{12} \\
    m^{(ij)}_{21} & m^{(ij)}_{22}
  \end{bmatrix}
  \begin{bmatrix}
    p_1 \\
    p_2
  \end{bmatrix}
  = \begin{bmatrix}
    m^{(ij)}_{11}p_1 + m^{(ij)}_{12}p_2 \\
    m^{(ij)}_{21}p_1 + m^{(ij)}_{22}p_2
  \end{bmatrix}
\end{equation}
for some matrix $M^{(ij)}$ with suitable constants $m^{(ij)}_{k\ell}$ chosen so that $M^{(ij)}\textbf{p} \in \Delta^1$. In this one-dimensional case we can, of course, drop $p_2$ and work with functions that operate on the scalar $p_1 = p$:
\begin{equation}
  f_{ij}(p) = m^{(ij)}_{11} p + m^{(ij)}_{12} (1-p)
  = (m^{(ij)}_{11} - m^{(ij)}_{12}) p + m^{(ij)}_{12}
\end{equation}
or, what amounts to the same thing,
\begin{equation}
  \label{appeq:linear-operators}
  f_{ij}(p) = a_{ij}p + b_{ij}
\end{equation}
for constants $a_{ij}$, $b_{ij}$ ($i,j=1,2$). These operators are thus affine in $p$.

A classical choice is to interpret response $R_1$ as reward and response $R_2$ as punishment, and to set
\begin{equation}
  \label{appeq:LRP}
  \left.
    \begin{aligned}
      f_{11}(p) &= (1 - \gamma ) p + \gamma \\
      f_{12}(p) &= (1 - \gamma ) p \\
      f_{21}(p) &= (1 - \gamma ) p \\
      f_{22}(p) &= (1 - \gamma ) p + \gamma
    \end{aligned}
      \right\}
  \quad\textnormal{i.e.}\quad
  \left.
    \begin{aligned}
      f_{11}(p) &= p + \gamma (1-p) \\
      f_{12}(p) &= p - \gamma p \\
      f_{21}(p) &= p - \gamma p \\
      f_{22}(p) &= p + \gamma (1-p)
    \end{aligned}
  \right\}
\end{equation}
for constant $0 < \gamma < 1$. In other words, the value of $p$ is augmented whenever grammar $G_1$ is rewarded or $G_2$ is punished, and decreased otherwise, as is evident from the representation on the right in \eqref{appeq:LRP}. The parameter $\gamma$, which governs how large or small modifications to $p$ the learner makes in response to the environment's responses, can be interpreted as a learning rate.

The operators in \eqref{appeq:LRP} constitute the one-dimensional linear reward--penalty scheme of \textcite{BusMos1955}, first applied to linguistic problems by \textcite{Yan2000}. To extend this model to cater for adult L2 acquisition, I now assume that adult L2 learners employ the same general learning strategy but are, additionally, subject to a bias which discounts grammars that are L2-difficult. As mentioned above, I take $G_1$ to refer to the grammar that incurs L2-difficulty, and assume $G_2$ not to be targeted by a similar bias. A simple extension of \eqref{appeq:LRP} is then the following, for secondary learning rate parameter $\delta$:
\begin{equation}
  \label{appeq:L2LRP}
  \left.
    \begin{aligned}
      f_{11}(p) &= (1 - \gamma - \delta ) p + \gamma \\
      f_{12}(p) &= (1 - \gamma - \delta ) p \\
      f_{21}(p) &= (1 - \gamma - \delta ) p \\
      f_{22}(p) &= (1 - \gamma - \delta ) p + \gamma
    \end{aligned}
      \right\}
  \quad\textnormal{i.e.}\quad
  \left.
    \begin{aligned}
      f_{11}(p) &= p + \gamma (1-p) - \delta p \\
      f_{12}(p) &= p - \gamma p - \delta p \\
      f_{21}(p) &= p - \gamma p - \delta p \\
      f_{22}(p) &= p + \gamma (1-p) - \delta p
    \end{aligned}
  \right\}.
\end{equation}
As is evident from the forms on the right, the additional term $-\delta p$ (with positive $\delta$) constitutes a negative bias experienced by grammar $G_1$, regardless of the environment's response. We require $0 \leq \delta \leq 1 - \gamma$ to guarantee $p$ always remains in the interval $[0,1]$.

\section{Learning: asymptotic results}

Learning, under this operationalization, is a stochastic process. In particular, it is (in practice) impossible to predict the exact evolution of $p$ given an initial state $p_0$. However, an explicit recursive solution exists for all moments of the distribution of $p$:
\begin{lemma}[Bush \& Mosteller \protect\citeyear{BusMos1955}:~98]
  \label{thm:bush-mosteller}
  For any linear operators of the form \eqref{appeq:linear-operators} in a stationary random environment, the following recursion holds for the $m$th moment of $p$, $\langle p^m \rangle$:
  \begin{equation}
    \label{appeq:moment-recursion}
    \langle p^m \rangle_{n+1} = \sum_{k=0}^m {m \choose k} \left( \Omega'_{m,k} \langle p^k \rangle_n + (\Omega_{m,k} - \Omega'_{m,k}) \langle p^{k+1} \rangle_n \right),
  \end{equation}
  where
  \begin{equation}
    \left.
      \begin{aligned}
        \Omega_{m,k} &= a_{11}^k b_{11}^{m-k} \omega_{11} + a_{12}^k b_{12}^{m-k} \omega_{12} \\
        \Omega'_{m,k} &= a_{21}^k b_{21}^{m-k} \omega_{21} + a_{22}^k b_{22}^{m-k} \omega_{22}
      \end{aligned}
          \right\}.
  \end{equation}
\end{lemma}
\noindent It will be useful to define the following averages (for $i=1,2$):
\begin{equation}
  \label{appeq:averages}
  \left.
    \begin{aligned}
      \overline{a}_i &= a_{i1}\omega_{i1} + a_{i2}\omega_{i2} \\
      \overline{b}_i &= b_{i1}\omega_{i1} + b_{i2}\omega_{i2} \\
      \overline{\overline{a}}_i &= a^2_{i1}\omega_{i1} + a^2_{i2}\omega_{i2} \\
      \overline{\overline{b}}_i &= b^2_{i1}\omega_{i1} + b^2_{i2}\omega_{i2} \\
      \overline{ab}_i &= a_{i1}b_{i1}\omega_{i1} + a_{i2}b_{i2}\omega_{i2}
    \end{aligned}
      \right\}.
\end{equation}
It is then a straightforward algebraic exercise to derive the following results concerning the first two raw moments from Lemma \ref{thm:bush-mosteller}:
  \begin{equation}
    \label{appeq:mean-recursion}
    \left.
      \begin{aligned}
        \langle p \rangle_{n+1} &= \overline{b}_2 + (\overline{b}_1 - \overline{b}_2 + \overline{a}_2) \langle p \rangle_n + (\overline{a}_1 - \overline{a}_2) \langle p^2 \rangle_n \\
        \langle p^2 \rangle_{n+1} &= \overline{\overline{b}} 
        + (\overline{\overline{b}}_1 - \overline{\overline{b}}_2 + 2\overline{ab}_2) \langle p \rangle_n 
        + (2\overline{ab}_1 - 2\overline{ab}_2 + \overline{\overline{a}}_2) \langle p^2 \rangle_n
        + (\overline{\overline{a}}_1 - \overline{\overline{a}}_2) \langle p^3 \rangle_n
      \end{aligned}
          \right\}.
  \end{equation}

Thus, in the general case, the $m$th moment depends on the $(m+1)$th moment. This problematic upward dependence disappears, however, if $\overline{a}_1 - \overline{a}_2 = 0$ and $\overline{\overline{a}}_1 - \overline{\overline{a}}_2 = 0$. Using the fact that $\omega_{i1} + \omega_{i2} = 1$, it is easy to check that this is the case if $a_{ij} = a$ for some common $a$, that is to say, if the slopes of the four affine operators are identical. This is obviously the case with both the classical linear reward--penalty scheme \eqref{appeq:LRP} as well as its extension to L2 learning \eqref{appeq:L2LRP}. For the mean, we then have
\begin{equation}
  \langle p \rangle_{n+1} = C_0 + C_1 \langle p \rangle_n
\end{equation}
with $C_0 = \overline{b}_2$ and $C_1 = \overline{b}_1 - \overline{b}_2 + \overline{a}_2$, a simple linear difference equation with solution
\begin{equation}
  \label{appeq:mean-solution}
  \langle p \rangle_n = C_1^n \langle p \rangle_0 + (1 - C_1^n) \langle p \rangle_{\infty}
\end{equation}
where
\begin{equation}
  \langle p \rangle_{\infty} = \frac{C_0}{1 - C_1}
\end{equation}
is the limit at $n \to \infty$ as long as $|C_1| < 1$. The latter inequality is easy to verify for algorithms \eqref{appeq:LRP} and \eqref{appeq:L2LRP} as long as the environment satisfies $0 < \omega_{ii} < 1$, which we can assume without loss of generality.
\begin{lemma}
  Assume the learning environment satisfies $0 < \omega_{ii} < 1$ for $i=1,2$. Then $|C_1| < 1$ for both \eqref{appeq:LRP} and \eqref{appeq:L2LRP}.
\end{lemma}
\begin{proof}
The constant $C_1$ has been defined as
\begin{equation}
  C_1 = \overline{b}_1 - \overline{b}_2 + \overline{a}_2.
\end{equation}
For algorithm \eqref{appeq:L2LRP} this becomes, using the definitions in \eqref{appeq:averages},
\begin{equation}
  C_1 = (\omega_{11} - \omega_{22})\gamma - \gamma + 1 - \delta.
\end{equation}
The learning rate parameters are assumed to satisfy $0 < \gamma < 1$ and $0 \leq \delta \leq 1 - \gamma$, which implies $1 - \delta \geq \gamma$. Hence
\begin{equation}
  C_1 \geq(\omega_{11} - \omega_{22}) \gamma.
\end{equation}
Since $\omega_{11}$ and $\omega_{22}$ are probabilities and we furthermore assume they belong to the open interval $]0,1[$, their difference satisfies $-1 < \omega_{11} - \omega_{22} < 1$. Hence
\begin{equation}
  C_1 > - \gamma > -1.
\end{equation}
On the other hand, since $\delta \geq 0$,
\begin{equation}
  C_1 \leq (\omega_{11} - \omega_{22}) \gamma - \gamma + 1
  < \gamma - \gamma + 1
  = 1.
\end{equation}
All in all, $|C_1| < 1$.
\end{proof}

Taking the constants $a_{ij}, b_{ij}$ from \eqref{appeq:L2LRP} we then have, after simplifying all the factors,
\begin{equation}
  \langle p \rangle_{\infty} = \frac{\omega_{22}}{\omega_{12} + \omega_{22} + d},
\end{equation}
where $d = \delta / \gamma$. Given that response $R_2$ was identified as punishment, the parameters $\omega_{12}$ and $\omega_{22}$ here refer to the probability of grammars $G_1$ and $G_2$ being punished, respectively. In line with previous work \autocite{Yan2000}, I will call these the \emph{penalty probabilities} associated with the two grammars, and will write $\pi_1 = \omega_{12}$ and $\pi_2 = \omega_{22}$ in what follows for simplicity. We have thus shown:
\begin{prop}
  \label{thm:means}
  For algorithm \eqref{appeq:L2LRP}, the expected value of $p$, the probability of use of grammar $G_1$, after an infinity of learning iterations is
  \begin{equation}
    \langle p \rangle_{\infty} = \frac{\pi_2}{\pi_1 + \pi_2 + d},
  \end{equation}
  where $\pi_1$ and $\pi_2$ are the penalty probabilities of the two grammars and $d = \delta / \gamma$ supplies the relative L2-difficulty of grammar $G_1$. The asymptotic expectation for algorithm \eqref{appeq:LRP} is obtained by setting $\delta = 0$ and thus $d = 0$.
\end{prop}

The second raw moment, in turn, evolves as
\begin{equation}
\langle p^2 \rangle_{n+1} = D_0 + D_1 \langle p^2 \rangle_n + D_2\langle p \rangle_n
\end{equation}
with $D_0 = \overline{\overline{b}}_2$, $D_1 = 2\overline{ab}_1 - 2\overline{ab}_2 + \overline{\overline{a}}_2$ and $D_2 = \overline{\overline{b}}_1 - \overline{\overline{b}}_2 + 2\overline{ab}_2$. Plugging the solution for the mean \eqref{appeq:mean-solution} in this equation we have
\begin{equation}
  \langle p^2 \rangle_{n+1} = D_0 + D_1 \langle p^2 \rangle_n + D_2(\langle p \rangle_0 - \langle p \rangle_{\infty}) C_1^n + D_2\langle p \rangle_{\infty}
\end{equation}
or, in other words,
\begin{equation}
\langle p^2 \rangle_{n+1} = E_0 + D_1 \langle p^2 \rangle_n + E_1 C_1^n,
\end{equation}
where $E_0 = D_0 + D_2\langle p \rangle_{\infty}$ and $E_1 = D_2( \langle p \rangle_0 - \langle p \rangle_{\infty})$. For large $n$, we ignore the term $E_1C_1^n$ since $|C_1| < 1$. Hence as $n \to \infty$, the second raw moment tends to the limit
\begin{equation}
  \langle p^2 \rangle_{\infty} = \frac{E_0}{1 - D_1} = \frac{D_0 + D_2\langle p \rangle_{\infty}}{1 - D_1}.
\end{equation}
It can further be shown that, for algorithms \eqref{appeq:LRP} and \eqref{appeq:L2LRP} and for fixed $d$, $\langle p^2 \rangle_{\infty} \to \langle p \rangle_{\infty}^2$ as $\gamma \to 0$, meaning that the limiting variance of $p$, $V[p]_{\infty} = \langle p^2 \rangle_{\infty} - \langle p \rangle_{\infty}^2$, converges to zero:
\begin{prop}
  \label{thm:variances}
  For both \eqref{appeq:LRP} and \eqref{appeq:L2LRP}, the variance of $p$ in the limit $n\to\infty$ can be made arbitrarily small by assuming a sufficiently small learning rate.
\end{prop}
\begin{proof}
We show the result for algorithm \eqref{appeq:L2LRP}; the statement for algorithm \eqref{appeq:LRP} follows as the special case $\delta = 0$. Recall that algorithm \eqref{appeq:L2LRP} consists of the statement that the constants of the linear operators satisfy $a_{ij} = a = 1 - \gamma - \delta$ (for $i,j=1,2$) and $b_{11} = b_{22} = \gamma$, $b_{12} = b_{21} = 0$, where $0 < \gamma \leq 1$ and $0 \leq \delta \leq 1 - \gamma$.

The difference in response probabilities $\omega_{22} - \omega_{11}$ will recur often in the following calculations; let $\omega = \omega_{22} - \omega_{11}$ for convenience.

Let $d = \delta/\gamma$ and assume $d$ is fixed, so that as $\gamma \to 0$, also $\delta \to 0$. Substituting the constants $a_{ij} = a$ and $b_{ij}$ in the definitions \eqref{appeq:averages} yields
\begin{equation}
  \left.
    \begin{aligned}
      D_0 &= \overline{\overline{b}}_2 = \gamma^2 \omega_{22} \\
      D_1 &= 2\overline{ab}_1 - 2\overline{ab}_2 + \overline{\overline{a}}_2 = a^2 - 2a\gamma\omega \\
      D_2 &= \overline{\overline{b}}_1 - \overline{\overline{b}}_2 + 2\overline{ab}_2 = 2a\gamma\omega_{22} - \gamma^2 \omega
    \end{aligned}
      \right\}.
\end{equation}
Moreover,
\begin{equation}
  1 - a^2 = 1 - (1 - \gamma - \delta)^2 = 2\gamma + 2\delta - \gamma^2 - 2\gamma\delta - \delta^2 = \gamma \left(2 + 2\frac{\delta}{\gamma} - \gamma - 2\delta - \frac{\delta^2}{\gamma} \right),
\end{equation}
in other words
\begin{equation}
  1 - a^2 = \gamma (2 + 2d - \gamma - 2\delta - d\delta).
\end{equation}
It now follows that
\begin{equation}
  1 - D_1 = 1 - a^2 + 2a\gamma \omega = \gamma (2 + 2d - \gamma - 2\delta - d\delta ) + 2a\gamma \omega.
\end{equation}
On the other hand
\begin{equation}
    \langle p^2 \rangle_{\infty} = \frac{D_0 + D_2\langle p \rangle_{\infty}}{ 1 - D_1},
\end{equation}
in other words,
\begin{equation}
    \langle p^2 \rangle_{\infty} = \frac{\gamma^2 \omega_{22} + (\gamma^2\omega + 2a\gamma\omega_{22}) \langle p \rangle_{\infty}}{\gamma (2 + 2d - \gamma - 2\delta - d\delta) + 2 a\gamma\omega}
    = \frac{\gamma\omega_{22} + (\gamma\omega + 2a\omega_{22}) \langle p \rangle_{\infty}}{2 + 2d - \gamma - 2\delta - d\delta + 2a\omega}.
\end{equation}
As $\gamma \to 0$, $\delta \to 0$ and $a \to 1$. Hence
\begin{equation}
  \lim_{\gamma \to 0} \langle p^2 \rangle_{\infty}
  = \frac{2\omega_{22}\langle p \rangle_{\infty}}{2 + 2d + 2\omega}
  = \frac{\omega_{22}}{1 + \omega + d}\langle p \rangle_{\infty}.
\end{equation}
On the other hand,
\begin{equation}
  1 + \omega = 1 - \omega_{11} + \omega_{22} = \omega_{21} + \omega_{22}.
\end{equation}
Recalling the notational convention $\pi_1 = \omega_{21}$ and $\pi_2 = \omega_{22}$, we now have, with the help of Proposition \ref{thm:means},
\begin{equation}
  \lim_{\gamma \to 0} \langle p^2 \rangle_{\infty}
  = \frac{\pi_2}{\pi_1 + \pi_2 + d} \langle p \rangle_{\infty}
  = \langle p \rangle_{\infty} \langle p \rangle_{\infty}
  = \langle p \rangle_{\infty}^2
\end{equation}
as desired.
\end{proof}

To recap, a population of learners employing either the linear reward--penalty scheme \eqref{appeq:LRP} or its L2 extension \eqref{appeq:L2LRP} will tend to a mean value of $p$ in the limit of large learning iterations which is given by Proposition \ref{thm:means}. Moreover, if learning is slow, so that the learning rates $\gamma$ and $\delta$ have small values, variability between learners in this population will be small. To be exact, that variability vanishes as $\gamma$ and $\delta$ tend to zero.

\section{Population dynamics}

We now concentrate on a learning environment characterized by the following penalty probabilities (see the main paper for motivation):
\begin{equation}
  \label{appeq:penalties}
  \left.
    \begin{aligned}
      \pi_{1} &= (1 - \sigma) \alpha_2 (1-p) + \sigma \alpha_2 (1-q) \\
      \pi_{2} &= (1 - \sigma) \alpha_1 p + \sigma \alpha_1 q
    \end{aligned}
      \right\},
\end{equation}
where $p$ and $q$ are the probabilities of grammar $G_1$ in the L1 and L2 speaker populations, respectively, $\sigma$ is the fraction of L2 speakers in the overall population, and $\alpha_1$ and $\alpha_2$ are the grammatical advantages of $G_1$ and $G_2$.

Making use of the asymptotic results from the previous section, we have the following general ansatz for inter-generational difference equations:
\begin{equation}
  \left.
    \begin{aligned}
      p_{n+1} - p_n &= \langle p_n \rangle_{\infty} - p_n \\
      q_{n+1} - q_n &= \langle q_n \rangle_{\infty} - q_n
    \end{aligned}
      \right\}.
\end{equation}
We may, without loss of generality, study the continuous-time limit
\begin{equation}
  \left.
    \begin{aligned}
      \dot p &= \langle p \rangle_{\infty} - p \\
      \dot q &= \langle q \rangle_{\infty} - q
    \end{aligned}
      \right\}
\end{equation}
instead. With Proposition \ref{thm:means}, this becomes
\begin{equation}
  \left.
    \begin{aligned}
      \dot p &= \frac{\pi_2 - (\pi_1 + \pi_2)p}{\pi_1 + \pi_2} = \frac{\pi_2(1-p) - \pi_1p}{\pi_1 + \pi_2} \\
      \dot q &= \frac{\pi_2 - (\pi_1 + \pi_2 + d)q}{\pi_1 + \pi_2 + d} = \frac{\pi_2(1-q) - \pi_1q - dq}{\pi_1 + \pi_2 + d}
    \end{aligned}
      \right\}.
\end{equation}
The denominators are strictly positive as long as $\alpha_1 \neq 0$ and $\alpha_2 \neq 0$, which is the case of interest here. They therefore do not contribute to the system's equilibria, and we may drop them without loss of generality. Doing this, filling in the penalties from \eqref{appeq:penalties}, and adopting the notational shorthand $\widetilde{x} = 1-x$ for any real $x$, we now have
\begin{equation}
    \left.
    \begin{aligned}
      \dot p &= \alpha_1 (\widetilde{\sigma} p + \sigma q) \widetilde{p} - \alpha_2 (\widetilde{\sigma} \widetilde{p} + \sigma \widetilde{q}) p \\
      \dot q &= \alpha_1 (\widetilde{\sigma} p + \sigma q) \widetilde{q} - \alpha_2 (\widetilde{\sigma} \widetilde{p} + \sigma \widetilde{q})q - d q
    \end{aligned}
      \right\}.
\end{equation}
Division of the right hand sides by $\alpha_2$, again without loss of generality, finally yields
\begin{equation}
  \label{appeq:the-system}
  \left.
    \begin{aligned}
      \dot p &= \alpha (\widetilde{\sigma} p + \sigma q) \widetilde{p} - (\widetilde{\sigma} \widetilde{p} + \sigma \widetilde{q}) p \\
      \dot q &= \alpha (\widetilde{\sigma} p + \sigma q) \widetilde{q} - (\widetilde{\sigma} \widetilde{p} + \sigma \widetilde{q} + D) q
    \end{aligned}
      \right\},
\end{equation}
where $\alpha = \alpha_1 / \alpha_2$ gives the ratio of the grammatical advantages and $D = d/\alpha_2$ represents the L2-difficulty of grammar $G_1$ scaled by the advantage of grammar $G_2$. It is easy to check (by examining the signs of $\dot p$ and $\dot q$ at the four sides of $[0,1]^2$) that this system is well-defined, in the sense that the unit square $[0,1]^2$ is forward-invariant under the dynamics.

Examination of \eqref{appeq:the-system} quickly shows that the origin $(p,q) = (0,0)$ is always an equilibrium of this system, for any selection of parameter values $\alpha$, $D$ and $\sigma$. A further, non-origin equilibrium may exist in $[0,1]^2$ depending on the combination of parameter values.
\begin{prop}\label{thm:equilibria}
  The system \eqref{appeq:the-system} has either one or two equilibria, for any combination of values of the parameters $\alpha$, $D$ and $\sigma$. The origin $(p,q) = (0,0)$ is always an equilibrium.
\end{prop}
\begin{proof}
The two nullclines of \eqref{appeq:the-system}, i.e.~the sets
\begin{equation}
  N_p = \{ (p,q) \in [0,1]^2 : \dot p = 0 \}
\end{equation}
and
\begin{equation}
  N_q = \{ (p,q) \in [0,1]^2 : \dot q = 0 \}
\end{equation}
are quadratic in $p$ and $q$ and define hyperbolas in the $pq$-plane in the general case. (An exception is the trivial case $\alpha = 1$, in which $N_p$ and $N_q$ reduce to straight lines intersecting at the origin.) In general, these hyperbolas will not have their centres at the origin, nor will their axes of symmetry be parallel to the coordinate axes. The hyperbolas may also, in the general case, intersect in up to four points in the real plane. Here, we show that at least one and at most two of those intersections occur in $[0,1]^2$.

Performing the relevant substitutions in \eqref{appeq:the-system}, it is quick to verify that the origin $(0,0)$ always belongs to both nullclines, and that $(1,1)$ always belongs to $N_p$. To show that $N_p$ and $N_q$ intersect in no more than two points in $[0,1]^2$, we examine the hyperbola $N_p$ in detail. Setting the first equation in \eqref{appeq:the-system} to zero, multiplying all terms out and rearranging, we have the canonical second-degree equation
\begin{equation}
  A_{pp}p^2 + 2A_{pq}pq + A_{qq}q^2 + B_pp + B_qq + C = 0
\end{equation}
with
\begin{equation}
  \label{appeq:hyperbola-coefficients}
  \left.
    \begin{aligned}
      A_{pp} &= \widetilde{\alpha}\widetilde{\sigma} \\
      2A_{pq} &= \widetilde{\alpha}\sigma \\
      A_{qq} &= 0 \\
      B_p &= \alpha\widetilde{\sigma} - 1 \\
      B_q &= \alpha\sigma \\
      C &= 0
    \end{aligned}
      \right\}.
\end{equation}
  The idea now is to show that the centre of the hyperbola, $(p_c, q_c)$, always satisfies either (i) $p_c < 0$ and $q_c > 1$, or (ii) $p_c > 1$ and $q_c < 0$, and that therefore one of its branches never intersects $[0,1]^2$. Using a translation of the coordinate system \autocite[see][246--247]{KelStr1968}, the centre is found to be at
  \begin{equation}
    (p_c, q_c) = \left(\frac{B_qA_{pq} - B_pA_{qq}}{2\Delta}, \frac{B_pA_{pq} - B_qA_{pp}}{2\Delta}\right),
\end{equation}
where $\Delta$ is the discriminant
 \begin{equation}
  \Delta = A_{pp}A_{qq} - A_{pq}^2.
\end{equation}
  With the coefficients \eqref{appeq:hyperbola-coefficients}, we find
\begin{equation}
  (p_c, q_c) = \left(\frac{-\alpha}{\widetilde{\alpha}}, \frac{\alpha\widetilde{\sigma} + 1}{\widetilde{\alpha}\sigma}\right).
\end{equation}
  Now, it is easy to check that, whenever $0 < \alpha < 1$, we have $p_c < 0$ and $q_c > 1$, and on the other hand that, when $\alpha > 1$, the conditions $p_c > 1$ and $q_c < 0$ obtain. Thus one of the branches of $N_p$ never touches the unit square $[0,1]^2$. On the other hand, as the other branch of $N_p$ always passes through both $(0,0)$ and $(1,1)$, and as $N_q$ always passes through $(0,0)$, it follows that $N_p$ and $N_q$ can intersect in at most one other point in $[0,1]^2$ in addition to the origin. In other words, the system \eqref{appeq:the-system} has either one or two equilibria in $[0,1]^2$.
\end{proof}

When the non-origin equilibrium exists, it is tedious to solve \eqref{appeq:the-system} for it explicitly in the general case. However, to understand the qualitative dynamics it suffices to study the conditions under which the equilibrium at the origin reverses stability, giving rise to the second equilibrium. To do this, we linearize about the origin, i.e.~inspect the eigenvalues of the system's Jacobian matrix at that point. This is
\begin{equation}
  J(p,q) = \begin{pmatrix}
    2\widetilde{\alpha}\widetilde{\sigma}p + \widetilde{\alpha}\sigma q + \alpha\widetilde{\sigma} - 1 & \widetilde{\alpha}\sigma p + \alpha\sigma \\
    \widetilde{\alpha}\widetilde{\sigma}q + \alpha\widetilde{\sigma} & \widetilde{\alpha}\widetilde{\sigma}p + 2\widetilde{\alpha}\sigma q + \alpha\sigma - D - 1
  \end{pmatrix}.
\end{equation}
Evaluated at the origin, the Jacobian reduces to
\begin{equation}
  J(0,0) = \begin{pmatrix}
    \alpha\widetilde{\sigma} - 1 & \alpha\sigma \\
    \alpha\widetilde{\sigma} & \alpha\sigma - D - 1
  \end{pmatrix}.
\end{equation}
An eigenvalue $\lambda$ must satisfy
\begin{equation}
  \begin{vmatrix}
    \alpha\widetilde{\sigma} - 1 - \lambda & \alpha\sigma \\
    \alpha\widetilde{\sigma} & \alpha\sigma - D - 1 - \lambda
  \end{vmatrix}
  = 0,
\end{equation}
which upon computation of the determinant yields the characteristic polynomial
\begin{equation}
  \lambda^2 + (D - \alpha + 2)\lambda + (D+1)(1 - \alpha\widetilde{\sigma}) - \alpha\sigma = 0.
\end{equation}
This has roots at
\begin{equation}
  \lambda_+ = \frac{\alpha - (D+2) + \sqrt{(\alpha + D)^2 - 4\alpha D \sigma}}{2}
\end{equation}
and
\begin{equation}
  \lambda_- = \frac{\alpha - (D+2) - \sqrt{(\alpha + D)^2 - 4\alpha D \sigma}}{2}.
\end{equation}
Since $0 \leq \sigma \leq 1$,
\begin{equation}
  (\alpha + D)^2 - 4\alpha D\sigma \geq (\alpha + D)^2 - 4\alpha D = \alpha^2 - 2\alpha D + D^2 = (\alpha - D)^2 \geq 0,
\end{equation}
and so both eigenvalues are always real. The origin is asymptotically stable if and only if both $\lambda_+ < 0$ and $\lambda_- < 0$.

If $\alpha > D + 2$, then $\lambda_+ > 0$, hence the origin is always unstable in this case. Thus let $\alpha < D + 2$, which covers the special case $\alpha < 1$ ($G_1$ less advantageous than $G_2$) but also a wide variety of empirically meaningful cases of $\alpha > 1$. Then $\lambda_- < 0$ always. For $\lambda_+$, we find $\lambda_+ < 0$ if and only if
\begin{equation}
  \sigma > \frac{(\alpha - 1)(D + 1)}{\alpha D} =: \sigma_{\textnormal{crit}}.
\end{equation}
%
We have thus found the following necessary and sufficient conditions for the total extinction of the L2-difficult grammar $G_1$ from both speaker populations:
\begin{prop}
  The system \eqref{appeq:the-system} has a unique (and stable) equilibrium $(p,q) = (0,0)$ if and only if
\begin{equation}
  \alpha < D + 2
  \quad\textnormal{\emph{and}}\quad
  \sigma > \sigma_{\textnormal{crit}} = \frac{(\alpha - 1)(D + 1)}{\alpha D}.
\end{equation}
\end{prop}

\end{document}